\documentclass{article}

\usepackage[round]{natbib}
\usepackage{times}
\usepackage{fancyhdr}

\usepackage[parfill]{parskip}

\usepackage{microtype}
\usepackage{graphicx}
\usepackage{subfig}
\usepackage{booktabs} 

\usepackage{fullpage}
\usepackage{courier}
\usepackage{authblk}
\RequirePackage{fancyhdr}

\usepackage[utf8]{inputenc} 
\usepackage[T1]{fontenc}    
\usepackage[hidelinks]{hyperref}       

\usepackage{amsmath}
\usepackage{amssymb}
\usepackage{mathtools}
\usepackage{amsthm}

\usepackage{url}            
\usepackage{booktabs}       
\usepackage{amsfonts}       
\usepackage{amssymb}       
\usepackage{mathtools}
\usepackage{nicefrac}       
\usepackage{microtype}      
\usepackage{graphicx}
\usepackage{xcolor}
\usepackage{enumitem}
\usepackage{xspace}
\usepackage[export]{adjustbox}

\usepackage{algorithm}
\usepackage{algorithmic}

\usepackage{multirow}
\usepackage{multicol}

\usepackage{floatpag}

\usepackage{color,array}

\theoremstyle{plain}
\newtheorem{theorem}{Theorem}[section]
\newtheorem{proposition}[theorem]{Proposition}
\newtheorem{lemma}[theorem]{Lemma}

\newtheorem{definition}[theorem]{Definition}
\newtheorem{assumption}[theorem]{Assumption}
\theoremstyle{remark}

\DeclareMathOperator*{\argmax}{arg\,max}
\DeclareMathOperator*{\argmin}{arg\,min}

\newcommand{\ee}{E}

\newcommand{\pp}{P}
\newcommand{\ppn}{\hat{P}_{n}}

\newcommand{\dd}{\mathrm{d}}

\newcommand{\hnorm}[1]{\left\| {#1} \right\|_\mathcal{H}}
\newcommand{\hdnorm}[1]{\left\| {#1} \right\|_{\mathcal{H}^\ast}}
\newcommand{\op}[1]{O_p\left( {#1} \right)}
\newcommand{\jpsi}{J_\psi}
\newcommand{\jh}{J_h}

\DeclareMathOperator*{\sumint}{%
\mathchoice%
  {\ooalign{$\displaystyle\sum$\cr\hidewidth$\displaystyle\int$\hidewidth\cr}}
  {\ooalign{\raisebox{.14\height}{\scalebox{.7}{$\textstyle\sum$}}\cr\hidewidth$\textstyle\int$\hidewidth\cr}}
  {\ooalign{\raisebox{.2\height}{\scalebox{.6}{$\scriptstyle\sum$}}\cr$\scriptstyle\int$\cr}}
  {\ooalign{\raisebox{.2\height}{\scalebox{.6}{$\scriptstyle\sum$}}\cr$\scriptstyle\int$\cr}}
}

\title{Geometry-Aware Instrumental Variable Regression}

\author{Heiner Kremer}
\author{Bernhard Sch{\"o}lkopf}

\affil{Max Planck Institute for Intelligent Systems, Tübingen, Germany}

\begin{document}
\floatpagestyle{plain}
\date{}
\maketitle
\begin{abstract}
\noindent
Instrumental variable (IV) regression can be approached through its formulation in terms of conditional moment restrictions (CMR). Building on variants of the generalized method of moments, most CMR estimators are implicitly based on approximating the population data distribution via reweightings of the empirical sample. While for large sample sizes, in the independent identically distributed (IID) setting, reweightings can provide sufficient flexibility, they might fail to capture the relevant information in presence of corrupted data or data prone to adversarial attacks. To address these shortcomings, we propose the Sinkhorn Method of Moments, an optimal transport-based IV estimator that takes into account the geometry of the data manifold through data-derivative information. We provide a simple plug-and-play implementation of our method that performs on par with related estimators in standard settings but improves robustness against data corruption and adversarial attacks.
\end{abstract}

\section{Introduction}
\begin{figure}[t]
    \centering
    \includegraphics[width=.65\linewidth]{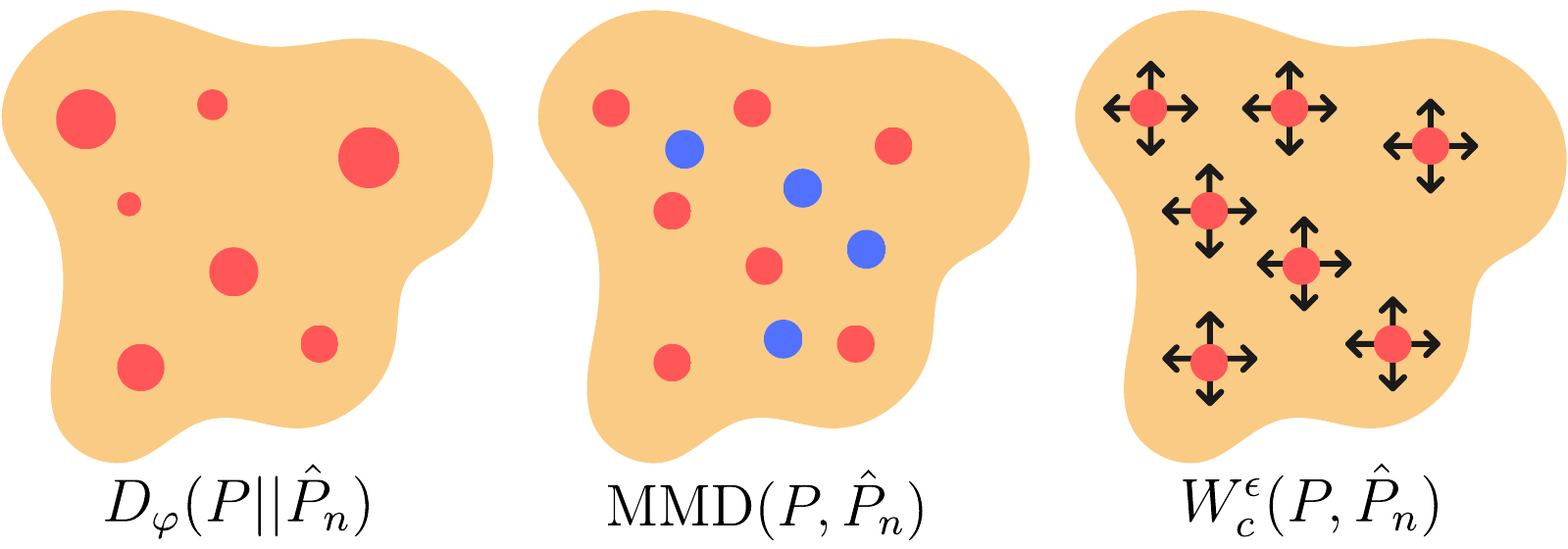}
    \caption{Paradigms to approximate $P_0$ from data (red dots) in the GEL framework. $\varphi$-divergence-based estimators (left) approximate $P_0$ by reweighting (weight $\hat{=}$ size) the sample (e.g., \citep{ai2003efficient,bennett2020variational}.
    MMD-based estimators (middle) allow for sampling additional data points (blue dots)~\citep{pmlr-v202-kremer23a}. In contrast, optimal transport-based estimators (right) allow to move around the data points (present work).}
    \label{fig:distances}
\end{figure}
Instrumental variable regression is one of the most widespread approaches for learning in presence of confounding~\citep{angrist2008mostly}.
It is applicable in situation where one is interested in inferring the outcome $Y$ of some treatment $T$, where both, treatment and outcome, are affected by a so-called unobserved confounder $U$.
To eliminate the confounding bias, one can take into account an instrumental variable $Z$, which i) affects the treatment $T$, ii) affects the outcome $Y$ only through its effect on $T$, and, iii) is independent of the confounder $U$. 
While traditionally the problem has been addressed through the 2-stage least squares approach~\citep{angrist2008mostly}, in recent years the formulation in terms of conditional moment restrictions (CMR) has gained popularity for its potential to benefit from advances in machine learning models~\citep{bennett2020deep,bennett2020variational,Dikkala20:Minimax,zhang2021maximum,muandet2020dual,kremer2022functional,pmlr-v202-kremer23a}.
The CMR formulation of IV regression is based on restricting the expectation of the prediction residual $Y-f(T)$ conditioned on the instruments $Z$, where $f$ denotes a causal relation from $T$ to $Y$ that one wants to infer.
In general, this leads to a zero-sum game in which one minimizes an objective with respect to the model parameters and maximizes it with respect to an adversary function that detects the moment violations~\citep{bennett2020deep,Dikkala20:Minimax}. 
One of the most general frameworks for learning with moment restrictions is the family of generalized empirical likelihood (GEL) estimators~\citep{owen2001empirical,qin-lawless,imbens1998,kitamura}, which includes the prominent generalized method of moments~\citep{hansen,hansen-finite-samples,hall2004generalized}. 
The idea behind empirical likelihood is to learn a model via maximum likelihood estimation without specifying a parametric form of the data distribution~\citep{owen2001empirical}.
In practice, this is realized by learning a non-parametric approximation of the population data distribution $P_0$ along with the model $f$ by means of minimizing a $\varphi$-divergence under the moment restrictions. 
However, by relying on $\varphi$-divergences one effectively restricts the estimator of the population distribution to reweightings of the sample. The reweighting assumption has recently been lifted by \citet{pmlr-v202-kremer23a} by introducing an estimator based on maximum mean discrepancy~\citep{gretton2012kernel}.
Their estimator allows for more fine-grained approximations of $P_0$ by sampling additional data points from a generative model.
While reweightings of the present data or sampling of additional points might be suitable to find sufficiently close approximations of the population distribution in some cases, in presence of highly complex data manifolds, e.g., image spaces, they might become ineffective as they are blind towards the geometry of the data space. This is particularly relevant in the presence of poisoned~\citep{chen2017targeted} or adversarial~\citep{goodfellow2014explaining} data points, i.e., data that has been corrupted with small perturbations which lead to vastly inaccurate predictions. The key to robustness against such perturbations is to look at how the learning signal changes around the empirical data points, i.e., to take into account the geometry of the signal with respect to the data manifold.
We implement the idea of a geometry-aware learning with conditional moment restrictions by proposing an empirical likelihood-type estimator based on a regularized optimal transport distance, which we call the Sinkhorn Method of Moments (SMM). Figure~\ref{fig:distances} schematically compares our method to previous approaches to empirical likelihood estimation.
\paragraph{Our contributions}
\begin{itemize}
    \item We propose the Sinkhorn Method of Moments (SMM), the first geometry-aware approach to IV regression resulting from an empirical likelihood-type estimator based on the Sinkhorn distance.
    \item We derive the dual form of our estimator and a leading order expansion that lets us compute our estimator with stochastic gradient methods.
    \item We show that under standard assumptions, our method is consistent for models identified via conditional moment restrictions.
    \item We derive a kernel-based implementation of our method that can be interpreted as a geometry-aware variant of a 2-stage generalized method of moments estimator for conditional moment restrictions. 
    \item Our experiments demonstrate that SMM is competitive with state-of-the-art IV estimators in standard settings and can provide an improvement in presence of corrupted data and adversarial examples.
\end{itemize}

Section~\ref{sec:3:background} introduces empirical likelihood estimation for conditional moment restrictions, followed by the derivation of our method in Section~\ref{sec:3:smm}. Empirical results are provided in Section~\ref{sec:3:results} and related work is discussed in Section~\ref{sec:3:related-work}.

\section{Empirical Likelihood Estimation for CMR} \label{sec:3:background}
In the following let $T$, $Y$ and $Z$ denote random variables taking values in $\mathcal{T} \subseteq \mathbb{R}^{d_t}$, $\mathcal{Y} \subseteq \mathbb{R}^{d_y}$ and $\mathcal{Z} \subseteq \mathbb{R}^{d_z}$ respectively.
We denote by $E_{P}[\cdot]$ the expectation operator with respect to a distribution $P$ and drop the subscript whenever we refer to the population distribution $P_0$.

Conditional moment restrictions identify a function of interest $f_0 \in \mathcal{F}$ by restricting the conditional expectation of a so-called moment function $\psi: \mathcal{T} \times \mathcal{Y} \times \mathcal{F} \rightarrow \mathbb{R}^m$,
\begin{align}
    E[\psi(T,Y;f_0) | Z] = 0 \ P_Z\mathrm{-a.s.}
    \label{eq:3:cmr}
\end{align}
The most prominent example of this problem is instrumental variable (IV) regression, where the moment function is given by the prediction residual $\psi(t,y;f) = y - f(t)$ and the conditioning variable $Z$ denotes the instrument. IV regression is one of the major practical approaches to deal with endogenous variables~\citep{Pearl2000:CMR} and has been largely adopted by the causal machine learning community~\citep{deepiv,xu2021learning,singh2019kernel,zhang2021maximum,saengkyongam2022exploiting}.

Learning with conditional moment restrictions is challenging mostly due to two factors. The first one is that equation~\eqref{eq:3:cmr} contains a \emph{conditional} expectation over the treatments $T$ and outcomes $Y$, while one generally has access to a sample from the \emph{joint} distribution over $(T,Y,Z) \sim P_0$. For a sufficiently complex data generating process the accurate estimation of a conditional distribution from the corresponding joint distribution can require large amounts of data~\citep{hall1999methods}.
This can be avoided by rewriting the CMR \eqref{eq:3:cmr} in terms of an equivalent \emph{variational} formulation~\citep{BIERENS1982105}
\begin{align}
    E[\psi(T,Y;f_0)^T h(Z)] = 0 \ \forall h\in\mathcal{H},
    \label{eq:3:vmr}
\end{align}
where $\mathcal{H}$ is a sufficiently rich function space, e.g., the space of square-integrable functions~\citep{BIERENS1982105} or the reproducing kernel Hilbert space of a certain kind of kernel~\citep{kremer2022functional}. While \eqref{eq:3:vmr} avoids the conditional expectation operator, it involves an infinite-dimensional over-determined system of equations. 
The second difficulty is the fact that the moment restrictions identify the function of interest $f_0$ via the \emph{population} distribution $P_0$ of the data, about which one usually only has partial information in terms of a sample $\mathcal{D} = \{(t_i,y_i,z_i) \}_{i=1}^n$ with empirical distribution $\ppn := \frac{1}{n} \sum_{i=1}^n \delta_{(t_i,y_i,z_i)}$, where $\delta_{(t_i,y_i,z_i)}$ denotes a point mass centered at $(t_i,y_i,z_i)$. While the true function $f_0$ is identified by the population moment restrictions \eqref{eq:3:vmr}, it might not satisfy the empirical counterpart of \eqref{eq:3:vmr} and thus one might not retrieve $f_0$ by enforcing it.
Empirical likelihood estimation \citep{owen88,owen90,qin-lawless} has been proposed as a flexible tool to solve over-determined moment restriction problems with access to only a finite sample. The idea is based on approximating the population distribution by seeking a distribution with minimal distance to the empirical one for which the moment restrictions can be fulfilled. We visualize this approach in Figure~\ref{fig:GEL}. The standard generalized empirical likelihood estimator~\citep{qin-lawless} with the extension to conditional moment restrictions of \citet{kremer2022functional} takes the form $f^{\mathrm{FGEL}} = \argmin_{f \in \mathcal{F}} R(f)$  with
\begin{align*}
    R(f) =  &\min_{\pp \in \mathcal{P}_n} D_\varphi(\pp || \ppn) \quad \mathrm{s.t.} \quad E_P[\psi(T,Y;f)^T h(Z)] = 0 \ \ \forall h \in \mathcal{H}
\end{align*}
where $D_\varphi(P||Q) = \int \varphi\left(\frac{\dd P}{\dd Q} \right) \dd Q$ denotes the $\varphi$-divergence between distributions $P$ and $Q$ and $\mathcal{P}_n = \{P \ll \ppn : E_P[1]=1 \}$ denotes the set of distributions that are absolutely continuous with respect to the empirical one, i.e., re-weightings of the data points.

\section{Sinkhorn Method of Moments} \label{sec:3:smm}
The goal of this work is to extend the idea of empirical likelihood estimation to optimal transport distances.
Before deriving the method, we provide a brief introduction to optimal transport. 
Consider the random variable $\xi := (T,Y,Z)$ taking values in $\Xi := \mathcal{T} \times \mathcal{Y} \times \mathcal{Z} \subseteq \mathbb{R}^{d_\xi}$, with $d_\xi = d_t + d_y + d_z$, and let $\mathcal{P}(\Xi)$ denote the space of probability distributions over $\Xi$.

\paragraph{Optimal Transport}
Optimal transport provides an intuitive way of comparing two distributions by means of measuring the minimum effort of transforming one to another by moving probability mass at a certain cost. 
Let $P \in \mathcal{P}(\Xi)$ and $Q \in \mathcal{P}(\Xi)$ denote two probability distributions over $\Xi$ with densities or probability mass functions (pmf) $p$ and $q$ respectively. 
Let $\Pi(P,Q) \subset \mathcal{P}(\Xi \times \Xi)$ denote the space of joint probability distributions over the product space $\Xi \times \Xi$ with marginals $P$ and $Q$. Define the projection operators $\mathbb{P}_{1}$ and $\mathbb{P}_{2}$ with $\mathbb{P}_1(x,y) = x$ and $\mathbb{P}_2(x,y) = y$ and their push-forward operation $\mathbb{P}_{i\sharp}$ such that for any element of $\Pi(P,Q)$, with density (or pmf) $\pi$ we have $\mathbb{P}_{1\sharp}\pi = \sumint \pi(\xi,\xi') \dd \xi' = p(\xi)$ and $\mathbb{P}_{2\sharp} \pi = \sumint \pi(\xi,\xi') \dd \xi = q(\xi')$. Then, for a cost function $c: \Xi \times \Xi \rightarrow \mathbb{R}$ we can define the Wasserstein distance between $P$ and $Q$ in the Kantorovich formulation as $W_c(P,Q) := \min_{\pi \in \Pi(P,Q)} \int c(\xi,\xi') \dd \pi(\xi,\xi')$. 
Computation of the Wasserstein distance requires the solution of an infinite-dimensional linear program. In order to enhance its computational efficiency, \citet{cuturi2013sinkhorn} proposed to modify the distance with an entropy regularization penalty, 
\begin{align*}
    W_c^\epsilon(P,Q) = \! \! \min_{\pi \in \Pi(P,Q)} \int c(\xi, \xi') \dd \pi(\xi, \xi') + \epsilon H(\pi | \mu \otimes \nu), 
\end{align*}
where the relative entropy between $\pi$ and a reference measure $\mu \otimes \nu \in \mathcal{P}(\Xi \times \Xi)$ is defined as 
\begin{align*}
    H(\pi| \mu \otimes \nu) = \int_{\Xi \times \Xi}\log\left( \frac{\dd \pi(\xi,\xi')}{\dd \mu(\xi) \dd \nu(\xi')} \right) \dd \pi(\xi,\xi').
\end{align*}
The resulting distance can be efficiently computed with the matrix scaling algorithm of \citet{sinkhorn1967concerning}, from where it derives its name, Sinkhorn distance. 
We refer to \citet{peyre2019computational} for a comprehensive introduction to computational optimal transport for machine learning.
In order to define an estimator for the conditional moment restriction problem \eqref{eq:3:cmr}, first, we resort to the functional formulation of \citet{kremer2022functional}. Let $\mathcal{H}$ denote a sufficiently rich space of functions such that equivalence between \eqref{eq:3:cmr} and \eqref{eq:3:vmr} holds. Then we define the moment functional $\Psi: \mathcal{T} \times \mathcal{Y} \times \mathcal{Z} \times \mathcal{F} \rightarrow \mathbb{R}^m$ via its action on $h\in\mathcal{H}$ as $\Psi(t,y,z;f)(h) = \psi(t,y;f)^T h(z)$. This lets us express the CMR \eqref{eq:3:cmr} in its equivalent functional form, $ \| E[\Psi(T,Y,Z;f)] \|_{\mathcal{H}^\ast} = 0$, where $\| \cdot \|_{\mathcal{H}^\ast}$ denotes the norm in the dual space $\mathcal{H}^\ast$ of $\mathcal{H}$.
\begin{figure}
    \centering
    \includegraphics[width=.5\linewidth]{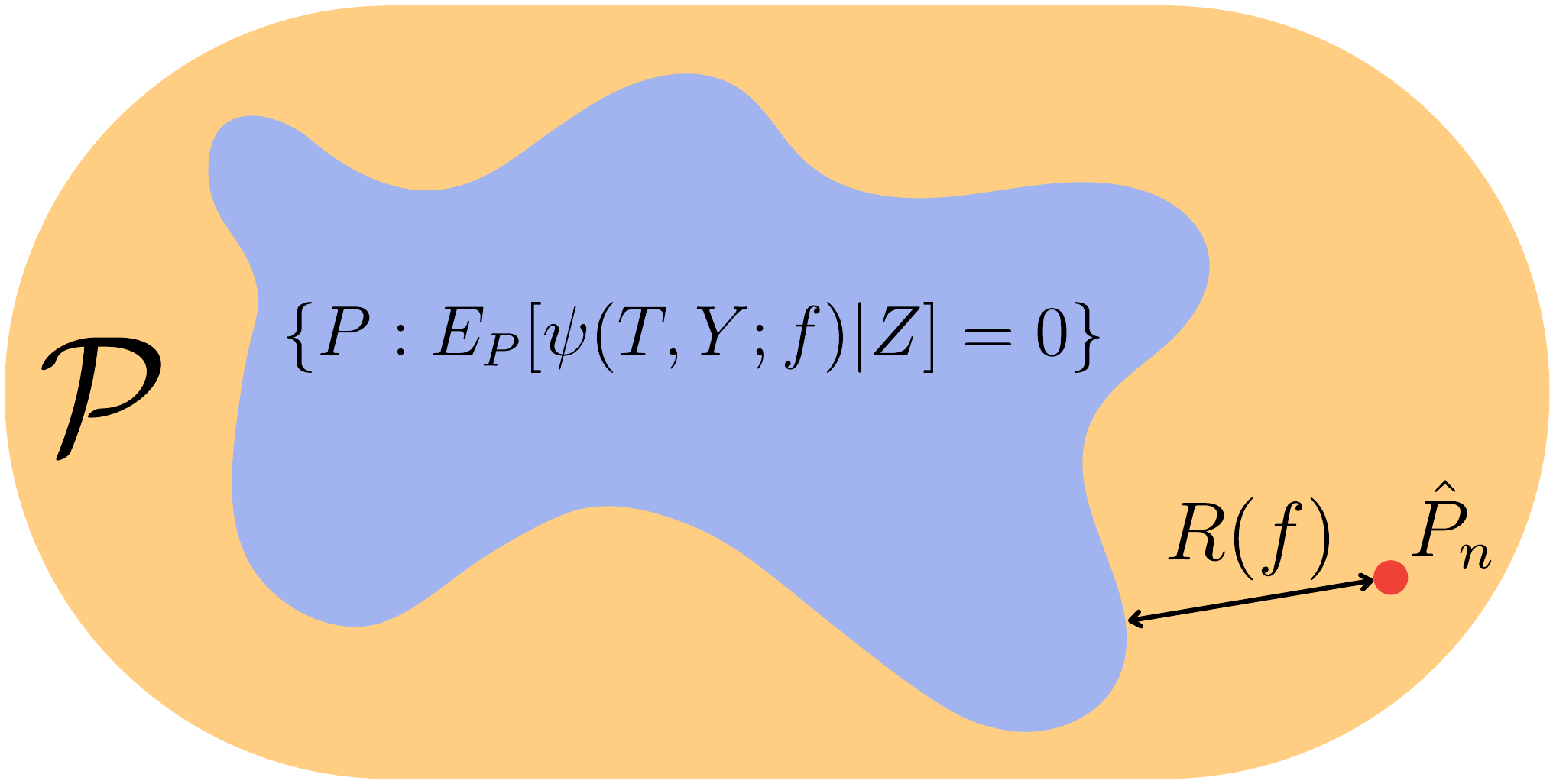}
    \caption{Sinkhorn profile. For every $f \in \mathcal{F}$, the Sinkhorn profile $R(f)$, \eqref{eq:3:primal}, is the minimal distance between the empirical distribution $\ppn$ and the set of distributions satisfying the CMR \eqref{eq:3:cmr}.}
    \label{fig:GEL}
\end{figure}

With this at hand, we can define the primal problem of the Sinkhorn Method of Moments estimator for conditional moment restrictions as the minimizer of the \emph{Sinkhorn profile} $R_\epsilon$ defined as
\begin{align}
   R_\epsilon(f) := &\min_{P \in \mathcal{P}} W_c^\epsilon(P, \ppn) \quad \mathrm{s.t.} \quad \| E_P[\Psi(T,Y,Z;f)] \|_{\mathcal{H}^\ast} = 0. \label{eq:3:primal} 
\end{align}
Using Lagrangian duality we can go over to the dual formulation of \eqref{eq:3:primal} as formalized by the following theorem whose proof is inspired by the mathematically closely related Sinkhorn Distributionally Robust Optimization (DRO) method of \citet{wang2023sinkhorn}.

\begin{theorem}[Duality] \label{th:3:duality}
Consider the Sinkhorn profile \eqref{eq:3:primal} with reference measure $\mu \otimes \nu \in \mathcal{P}(\Xi \times \Xi)$. Then \eqref{eq:3:primal} has the strongly dual form $R_\epsilon(f) = \sup_{h \in \mathcal{H}} D(f, h)$, where 
    \begin{align}
        D(f, h) :=
         \ee_{\xi' \sim \nu}\!\left[-\epsilon \log \ee_{\xi\sim\mu}\!\left[ e^{-\Psi(\xi;f)(h) -c(\xi,\xi')/{\epsilon}} \right] \right] \! . \label{eq:3:dual}
    \end{align}
\end{theorem}
In contrast to its original purpose, in our application, the goal of the entropic regularization penalty is not to make computation of the distance more efficient but rather to arrive at a relaxed dual problem~\eqref{eq:3:dual}. 
The dual Sinkhorn profile~\eqref{eq:3:dual} contains expectation operators with respect to the reference distributions $\mu$ and $\nu$ combined in a non-linear way. This casts optimization of the objective difficult as stochastic gradient estimates will be biased. 
One way to proceed is to resort to de-biasing techniques as discussed by \citet{wang2023sinkhorn} for their related DRO objective. However, on top of the problem of gradient estimation, computation of \eqref{eq:3:dual} requires sampling from two reference distributions $\mu$ and $\nu$ such that accurate gradient estimation becomes costly.

To avoid these issues, we propose an alternative solution for a special choice of reference measures and cost function. \citet{cuturi2013sinkhorn} chooses the reference measure as the product of the marginals of the coupling distribution $\pi$. For $W_c^\epsilon(P,Q)$ this corresponds to the choice $\mu \otimes \nu = P \otimes Q$.
The choice of $\mu$ and $\nu$ can be interpreted as a prior for distributions $P$ and $Q$ respectively. Motivated by this, we choose $\nu = \ppn$ and in order not to restrict the form of $P$ we use an uninformative prior and choose $\mu$ as the Lebesgue measure.   

The second modeling choice is the transport cost function $c$. Here, we use a weighted Euclidean norm,
\begin{align}
    c(\xi, \xi') &:= \frac{1}{2} (\xi - \xi')^T \Gamma (\xi - \xi') \label{eq:3:cost}\\
    &= \frac{1}{2}
\sum_{w \in \{t,y,z\}} \gamma_w \|w - w' \|_2^2, \nonumber
\end{align}
where the factors $\gamma_w > 0$ determine the transport cost in the spaces $\mathcal{T}$, $\mathcal{Y}$ and $\mathcal{Z}$ and we defined the block diagonal matrix $\Gamma: = \operatorname{diag}(\{\gamma_t I_{d_t}, \gamma_y I_{d_y}, \gamma_z I_{d_z} \}) \in \mathbb{R}^{d_\xi \times d_\xi}$, with $I_{d_i}$ denoting the identity matrix in $\mathbb{R}^{d_i}$. 
With these choices, the objective~\eqref{eq:3:dual} becomes
\begin{align}
    D(f,h)  = E_{\xi' \sim \ppn} \left[ - \epsilon \log E_{\xi \sim \mathcal{N}(\xi', \epsilon \Gamma^{-1})} \left[ e^{-\Psi(\xi;f)(h)} \right]\right], \label{eq:3:linear1}
\end{align}
where $\mathcal{N}(\xi', \epsilon \Gamma^{-1})$ denotes a multivariate Gaussian centered at $\xi' = (t',y',z')$ with diagonal covariance $\epsilon \Gamma^{-1}$. 
Thus, for each value of $\xi'$ we need to carry out an expectation over the moment violation $\exp(-\Psi(\xi;f)(h))$ with respect to a narrow Gaussian distribution centered at $\xi'$.
Now, as $\epsilon$ is a small regularization parameter, the integrand will only provide relevant contributions in a neighborhood of $\xi'$ and thus, for a sufficiently smooth moment function $\psi$ and instrument function $h$, we can employ a Taylor expansion and carry out the Gaussian expectation over $\xi$ in closed form. 
In the following, we define the weighted Laplacian $\Delta_\xi = \nabla_\xi \cdot \left( \Gamma^{-1} \nabla_\xi \right) = \sum_{w \in \{t,y,z \}} \frac{1}{\gamma_w} \Delta_w$ and the weighted $l_2$-norm $\|\cdot \|_{\Gamma}$ as $\| v \|_{\Gamma}^2 = v^T \Gamma^{-1}v$ for $v \in \mathbb{R}^{d_\xi}$.
\begin{theorem} \label{th:3:linearization}
Let the moment functional $\Psi(\cdot; f): \Xi \rightarrow \mathcal{H}^\ast$ be continuously differentiable everywhere for any $f \in \mathcal{F}$. 
Consider the SMM estimator with transport cost function \eqref{eq:3:cost} and reference measure $\ppn \otimes L$, where $L$ denotes the Lebesgue measure over $\Xi$. 
Then, for $\epsilon / \gamma_{i}$, $i \in [t,y,z]$, sufficiently small, up to constants and rescalings the objective of the dual Sinkhorn profile \eqref{eq:3:dual} takes the form
    \begin{align}
        D(f,h) =& E_{\xi \sim \ppn}\left[\left(I + \frac{\epsilon}{2} \Delta_\xi \right)\Psi(\xi;f)(h) \right]  - \frac{\epsilon}{2} E_{\xi \sim \ppn}\left[ \| \nabla_\xi \Psi(\xi;f)(h) \|_{\Gamma}^2 \right] + O(\epsilon^{3/2}).   \label{eq:3:smm-cue}
    \end{align}
\end{theorem}
Motivated by the classical 2-stage generalized method of moments (GMM) estimator \citep{hansen} we define the Sinkhorn Method of Moments by substituting the instrument function in the second term in \eqref{eq:3:smm-cue} by a first stage estimate $\tilde{f}$. We will show below that this does not harm the consistency and convergence properties of our method. Additionally, we add regularization on the instrument function $-\frac{\lambda}{2} \|h \|_\mathcal{H}^2$ to ensure that the optimization over $h$ is well behaved on finite samples.
\begin{definition}[SMM] \label{def:3:smm}
    Let $\tilde{f} \in \mathcal{F}$ denote a first-stage estimate of $f_0 \in \mathcal{F}$, then we define the Sinkhorn Method of Moments (SMM) estimator as the solution of the saddle-point problem 
    \begin{align}
    f^{\mathrm{SMM}} = \argmin_{f \in \mathcal{F}} \max_{h \in \mathcal{H}} M(f,h) - \epsilon \mathcal{R}(\tilde{f},h) \label{eq:3:smm}
    \end{align}
    with 
    \begin{align*}
    M(f,h) &= E_{\ppn}\left[\left( I+\frac{\epsilon}{2} \Delta_\xi \right) \Psi(\xi;f)(h) \right] \\
    \mathcal{R}(\tilde{f},h) &= \frac{1}{2} E_{\ppn}\left[ \| \nabla_\xi \Psi(\xi;\tilde{f})(h) \|_{\Gamma}^2 \right]  + \frac{\lambda}{2\epsilon} \| h \|_{\mathcal{H}}^2,
    \end{align*}
    where as before $\Psi(\xi;f)(h) = \psi(t,y;f)^T h(z)$.
\end{definition}
By using the 2-stage GMM-style estimator we shift most of the computational complexity into the optimization of the instrument function $h \in \mathcal{H}$. The optimization over the possibly high-dimensional model remains simple and even is a convex program, whenever $f$ has a convexity preserving parameterization, e.g., for linear models. In practice, if \eqref{eq:3:smm} is optimized with stochastic gradient methods, one can dynamically update the first stage estimate $\tilde{f}$ using the result from the previous iteration.
In the context of CMR estimation this GMM-inspired two stage procedure is a popular approach to stabilize the training~\citep{bennett2020deep,lewis2018adversarial,bennett2020variational}.
Note that without the 2-stage adaptation we would obtain an estimator similar in spirit to the continuous updating GMM estimator of \citet{hansen-finite-samples} or the FGEL estimator of \citet{kremer2022functional}, which can be harder to train in practice~\citep{hall2004generalized}.
The objective \eqref{eq:3:smm} involves a gradient and a Laplacian with respect to the data, which allows the method to take into account the geometry of the moment violation with respect to the data manifold. As we maximize the objective over $h \in \mathcal{H}$, we promote instrument functions which correspond to local minima of the moment violation $\psi(t,y;f)^T h(z)$ with respect to the data. 
Generally for CMR estimators the instrument function is responsible for translating the data into a learning signal for the model $f$. Choosing $h$ in a local minimum w.r.t.\ the data means that we attribute less importance to data points that lead to large increases in the moment violation when perturbed slightly. This makes the model less vulnerable to poisoned data and adversarial attacks. 
SMM's property to take into account how the learning signal changes in proximity of the data is unique compared to related estimators which are blind towards the geometry of the data manifold as they are based on reweighting the existing data \citep{bennett2020deep,lewis2018adversarial,Dikkala20:Minimax,bennett2020variational,kremer2022functional} or sampling additional data points~\citep{pmlr-v202-kremer23a} respectively.

\subsection{Consistency} \label{sec:3:consistency}
The following assumptions allow us to guarantee consistency and derive a convergence rate of our 2-stage estimator \eqref{eq:3:smm} in the parametric, uniquely identified setting. Suppose there exists a unique parameter $\theta_0 \in \Theta \subseteq \mathbb{R}^p$ for which $E[\psi(T,Y;\theta_0)|Z] = 0 \ P_Z\mathrm{-a.s.}$. In the following, let $x \in \mathcal{X} \subseteq \mathbb{R}^{d_x}$ denote the concatenation of $(t,y) \in \mathcal{T} \times \mathcal{Y}$ and let $i \in [m]$ be a shorthand for $i \in \{1, \ldots, m \}$. Further, we define the Jacobian of a vector-valued function $\psi: \mathcal{X} \times \Theta \rightarrow \mathbb{R}^m$ as $J_x \psi (x;\theta) \in \mathbb{R}^{m \times d_x}$.

\begin{assumption}[Identifiability] \label{as2:psi}
    $\theta_0 \in \Theta$ is the unique solution to $E[\psi(X;\theta)| Z] = 0 \ P_Z\mathrm{-a.s.}$; $\Theta$ is compact; $\psi(X;\theta)$ is continuous in $\theta$ everywhere w.p.1.
\end{assumption}
This is a standard assumption that provides identifiability of the true parameter $\theta_0$.

\begin{assumption}[Data regularity] \label{as9:data}
    The space $\Xi = \mathcal{T} \times \mathcal{Y} \times \mathcal{Z} \subset \mathbb{R}^{d_\xi} $ is compact.
\end{assumption} 

\begin{assumption}[Smoothness w.r.t.\ data] \label{as4:regularity}
    The moment function $\psi(\cdot;\theta): \mathcal{T} \times \mathcal{Y} \rightarrow \mathbb{R}^m$ is $C^{\infty}$-smooth in the data for every $\theta \in \Theta$.
    Further the sets of functions $\{ \psi(\cdot;\theta)_l : \theta \in \Theta \}$ and $\{\left(J_x \psi(\cdot;\theta) \right)_{lr} : \theta \in \Theta \}$, are $P_0$-Donsker for every $l \in [m]$ and $r \in [d_x]$. 
\end{assumption}
Assumption~\ref{as9:data} and \ref{as4:regularity} ensure that the moment function and its derivatives are well-behaved with respect to the data. While the compactness of the data space might be violated in practice, usually one can construct a sufficiently large compact set that contains the data with high probability. 

\begin{assumption} \label{as7:covariance}
    The matrix $V(Z;\theta) \in \mathbb{R}^{m \times m}$ defined as
    \begin{align}
        V(Z;\theta) = E[J_x \psi (X;\theta) \Gamma^{-1} J_x \psi(X;\theta)^T |Z] \label{eq:3:gradient-cov}
    \end{align}
    is non-singular for $\theta \in \{\theta_0, \bar{\theta}\}$ w.p.1, where $\bar{\theta}$ is an initial parameter estimate defined in Assumption~\ref{as6:first-estimate}.
\end{assumption}
This corresponds to the common assumption of a non-singular covariance matrix required by related estimators~\citep{newey04,bennett2020variational,kremer2022functional}, 
but, here, imposed on the covariance of the data-Jacobian.

\begin{assumption}[Instrument function] \label{as1:instrument}
    $\mathcal{H} = \bigoplus_{l=1}^m \mathcal{H}_i$ is a a sufficiently rich space of vector-valued functions such that equivalence between \eqref{eq:3:cmr} and \eqref{eq:3:vmr} holds. Further for $l \in [m]$, $h \in \mathcal{H}_l$ is $C^\infty$-smooth and the unit ball $\mathcal{H}_{l,1} := \{h \in \mathcal{H}_l : \| h\|_{\mathcal{H}_l} \leq 1 \}$ as well as $\{J_z h: h \in \mathcal{H}_{l,1} \}$ are $P_0$-Donsker. 
\end{assumption}
This is fulfilled, for example, by choosing each $\mathcal{H}_l$ as the RKHS of a universal, integrally stricly positive definite kernel, e.g., the Gaussian kernel, which we will formalize later. For neural network instrument function classes, equivalence between the variational and conditional formulations can be shown on basis of universal approximation theorems~\citep{yarotsky2017error,yarotsky2018optimal}.
In this case $\mathcal{H}_{l,1}$, $C^\infty$-smoothness can be realized by using smooth activation functions.

\begin{assumption}[Regularization] \label{as6:first-estimate}
One has access to a consistent first-stage parameter estimate $\bar{\theta}_n \overset{p}{\rightarrow} \bar{\theta}$ for which $E\left[\| \psi(X;\bar{\theta}_n) - \psi(X;\bar{\theta}) \|_\infty \right] = O_p(n^{-\zeta})$ and  $E \left[ \|J_x\psi(X;\bar{\theta}_n) - J_x\psi(X;\bar{\theta}) \|_\infty \right] = O_p(n^{-\zeta})$ with $0<\zeta \leq 1/2$. Choose $\lambda_n = O_p(n^{-\rho})$ with $0 < \rho < \zeta$.
\end{assumption}

For linear IV regression this implies $\|\bar{\theta}_n - \bar{\theta} \|_\infty = O_p(n^{-\zeta})$, which means $\bar{\theta}_n$ has to be a $n^{-\zeta}$-consistent estimator for $\bar{\theta}$, which can be any parameter for which \eqref{eq:3:gradient-cov} is non-singular, e.g., the true parameter $\theta_0$.

\begin{assumption}[Smoothness w.r.t.\ $\theta$] \label{as10:smooth-param}
    $\theta_0 \in \operatorname{int}(\Theta)$; 
    $\psi(x;\theta)$ is continuously differentiable in a neighborhood $\bar{\Theta}$ of $\theta_0$; and 
    $E[\sup_{\theta \in \bar{\Theta}} \| J_\theta \psi(X;\theta) \|^2 | Z] < \infty$ w.p.1; $\operatorname{rank}\left(E[J_\theta \psi(X;\theta_0) |Z] \right) = p$, w.p.1.
\end{assumption}
This smoothness assumption allows us to translate the convergence rate of the moment functional into a convergence rate for the parameter estimate.

With that, we are ready to state the consistency theorem for our estimator.

\begin{theorem}[Consistency] \label{th:3:consistency}
    Let Assumptions~\ref{as2:psi}-\ref{as6:first-estimate} be satisfied. For any $0 < \epsilon_1 < \epsilon_2$, choose $\epsilon \sim \operatorname{Uniform}([\epsilon_1, \epsilon_2])$.
    Then the SMM estimator $\hat{\theta}$ converges to the true parameter $\theta_0$ in probability $\hat{\theta} \overset{p}{\rightarrow} \theta_0$.
    
    If additionally Assumption~\ref{as10:smooth-param} is satisfied, then $\|\hat{\theta} - \theta_0 \| = O_{p}(n^{-1/2})$.
\end{theorem}
The consistency result is independent of the choice of instrument function space $\mathcal{H}$ as long as it fulfills Assumption~\ref{as1:instrument}. We now discuss two different implementations based on kernel methods and neural networks.

\subsection{Kernel-SMM} \label{sec:3:kernel-smm}
Choosing $\mathcal{H}$ as the RKHS of a suitable kernel, we can guarantee equivalence between the conditional and variational moment restrictions formulations \eqref{eq:3:cmr} and \eqref{eq:3:vmr}. On top of that, for RKHS instrument functions we can employ a representer theorem and carry out the optimization over the instrument function $h \in \mathcal{H}$ in closed form. The resulting estimator can be obtained as the solution of a simple minimization problem bearing close resemblance to the optimally weighted 2-stage GMM estimator but taking into account the geometry of the moment violation with respect to the data. Before deriving the result we provide the necessary background on reproducing kernel Hilbert spaces (RKHS).

\paragraph{Reproducing kernel Hilbert space}
An RKHS $\mathcal{H}$ is a Hilbert space of functions $h: \mathcal{Z} \rightarrow \mathbb{R}$ in which point evaluation is a bounded functional. With every RKHS one can associate a positive semi-definite kernel $k(\cdot, \cdot): \mathcal{Z} \times \mathcal{Z} \rightarrow \mathbb{R}$ with the reproducing property, i.e., for any $h \in \mathcal{H}$ we have $h(z) = \langle h, k(z,\cdot) \rangle_\mathcal{H}$. A kernel is called universal if its RKHS is dense in the set of all continuous real-valued functions~\citep{universalkernel06}. Further, a kernel is called integrally strictly positive definite (ISPD) if for any $h \in \mathcal{H}$ with $0 <\| h \|_\mathcal{H}^2 < \infty$, we have $\int_\mathcal{Z} h(z) k(z,z') h(z') \dd z \dd z' >0$. Refer to, e.g., \citet{scholkopf2002learning} and \citet{berlinet2011reproducing} for comprehensive introductions.

The following proposition specifies the properties of an RKHS for which Assumption~\ref{as1:instrument} is satisfied. 
\begin{proposition} \label{prop:3:rkhs}
    Let $\mathcal{Z} \subset \mathbb{R}^{d_z}$ be compact. Then,
    the instrument function space $\mathcal{H} = \bigoplus_{l=1}^m \mathcal{H}_l$, where each $\mathcal{H}_l$ corresponds to the RKHS of universal, integrally strictly positive definite kernel $k_l$, $l \in [m]$ fulfills Assumption~\ref{as1:instrument}.
\end{proposition}
Now, for a representer theorem to hold, in the following, we place infinite cost $\gamma_z = \infty$ on the transport of $z \in \mathcal{Z}$, i.e., we fix the instruments at their empirical locations. 
As long as $\gamma_t, \gamma_y < \infty$ this still allows for varying the functional relation between $Z$ and $T$ as well as $T$ and $Y$ in the training data.
In the following, define the block-diagonal matrix $\Gamma_x: = \operatorname{diag}(\{\gamma_t I_{d_t}, \gamma_y I_{d_y} \}) \in \mathbb{R}^{d_x}$ and the weighted Laplace operator $\Delta_x = \nabla_x \cdot \left(\Gamma_x^{-1} \nabla_x \right)$.
\\

\begin{theorem}[Kernel-SMM] \label{th:3:kernel-smm}
    Let $\mathcal{H} = \bigoplus_{l=1}^m \mathcal{H}_l$ be the direct sum of $m$ reproducing kernel Hilbert spaces with kernels $k_l: \mathcal{Z} \times \mathcal{Z} \rightarrow \mathbb{R}$.
    Let $\tilde{f} \in \mathcal{F}$ denote a first stage estimate of $f_0$ and let $\gamma_z = \infty$.
    Define $\psi_\Delta(f) \in \mathbb{R}^{nm}$, $L \in \mathbb{R}^{nm \times nm}$ and $Q(f) \in \mathbb{R}^{nm \times nm}$ with entries
    \begin{align*}
        \psi_\Delta(f)_{i \cdot l} =& \left( I + \frac{\epsilon}{2} \Delta_x \right)\psi_l(x_i;f)  \\
        L_{(i\cdot l), (j \cdot r)} =& \delta_{lr}  k_l(z_i,z_j) \\ 
        Q(f)_{ (i\cdot l) , (j \cdot r)} =& \frac{1}{n} \sum_{k=1}^n \sum_{s=1}^{d_x} \Big\{  k_l(z_i, z_k) \nabla_{x_s} \psi_l(x_k;{f})  
         \left(\Gamma_x^{-1}\right)_{ss} \nabla_{x_s} \psi_r(x_k; f) k_r(z_k, z_j) \Big\}.
    \end{align*}
    Then the Sinkhorn profile is given by
    \begin{align}
        R_{Q(\tilde{f})}(f) = \frac{1}{2n^2}  \psi_\Delta(f)^T L \left(Q(\tilde{f}) + \frac{\lambda}{\epsilon} L \right)^{-1} L \psi_\Delta(f). \label{eq:3:kernel-smm}
    \end{align}
\end{theorem}
Compared to the general saddle point formulation \eqref{eq:3:smm} the kernelized version \eqref{eq:3:kernel-smm} has the significant advantage that it only involves a minimization over the model parameters and thus avoids the difficulties of mini-max optimization~\citep{daskalakis2018training}. Algorithm~\ref{alg:kernel-smm} details the implementation of the multi-stage Kernel-SMM approach. In order to minimize the number of hyperparameters, we implement the gradient descent step with the limited memory BFGS method~\citep{liu1989limited}. We empirically observed that the $n$-step estimator effectively converges with the second iteration. 

\begin{algorithm}[t]
   \caption{$n$-stage Kernel-SMM}
   \label{alg:kernel-smm}
\begin{algorithmic}
  \STATE {\bfseries Input:} Initial function $\tilde{f}$, hyperparameters $\epsilon$, $\lambda$, $\gamma_x$  
  \FOR{$i=1,\ldots,n$}
    \STATE Compute $Q(\tilde{f})$
    \WHILE{not converged}
        \STATE $f \gets \operatorname{GradientDescent}(f, \nabla_f R_{Q(\tilde{f})}(f))$
    \ENDWHILE
        \STATE $\tilde{f} \gets f$
  \ENDFOR
  \STATE {\bfseries Output:} Function estimate $f$
\end{algorithmic}
\end{algorithm}

\subsection{Neural-SMM}
A particularly interesting alternative choice of instrument function space are neural network classes, as they can represent highly flexible functions while allowing for optimization via mini-batch stochastic gradient methods.
As demonstrated by related works~\citep{lewis2018adversarial,bennett2020deep,kremer2022functional}, such neural network-based approaches can lead to powerful and scalable estimators that may outperform the corresponding kernel method on large samples.
On the downside, they tend to be difficult to train due to the instability and hyperparameter sensitivity of mini-max optimization. This is particularly problematic for IV regression, as in contrast to standard supervised learning, it is non-trivial to define suitable validation metrics to set these hyperparameters. As a result, compared to \eqref{eq:3:kernel-smm}, those estimators require more attention and careful evaluation which makes them less suitable as plug-and-play IV estimators for practitioners. As the primary focus of this work is to introduce a new geometry-aware learning paradigm for IV regression independent of the instrument function class, we consider the simpler kernel version in the following and defer results for the Neural-SMM estimator to Appendix~\ref{app:3:additional}.

\section{Experimental Results} \label{sec:3:results}
We benchmark the kernel version of our method against a selection of plug-and-play IV estimators including maximum moment restrictions (MMR) \citep{zhang2021maximum}, sieve miniminum distance (SMD)~\citep{ai2003efficient} as well as the kernel variational method of moments (VMM)~\citep{bennett2020variational}. Results for the neural network version and related estimators can be found in Appendix~\ref{app:3:additional}.
For all kernel methods we choose a radial basis function kernel $k(z,z') = \exp(-\eta \| z- z'\|_2^2)$, where we set $\eta$ according to the median heuristic~\citep{garreau2018large}. The remaining hyperparameters of all methods are set by using the MMR objective on a validation data set (see Appendix~\ref{app:3:details}). 
In all experiments we consider perturbations in the treatment variable $t$ and fix the other variables at their empirical values by setting $\gamma_y, \gamma_z = \infty$ for SMM. Implementations of our estimators are available at \url{https://github.com/HeinerKremer/sinkhorn-iv/}.
\paragraph{IV Regression with Corrupted Data}
We consider the SimpleIV experiment of \citet{bennett2020variational} with the following data generating process,
\begin{align}
    &Z = \sin(\pi Z_0 / 10)  \label{eq:3:exp1}\\
    &T = -0.75 Z_0 + 3.5 H + 0.14 \eta - 0.6 \nonumber \\
    &Y = f(T;\theta_0) - 10 U + 0.1 \eta_2 \nonumber 
\end{align}
where $\eta_1, \eta_2, U \sim N(0,I)$ and $Z_0 \sim \operatorname{Uniform}([-5, 5])$. The model is given by $f(t;\theta) = \theta^1 t^2 + \theta^2 t + \theta^3$ with $\theta_0 = [3.0, -0.5, 0.5]$. This is a typical IV problem, where the unobserved confounder $U$ induces a non-causal dependence between $T$ and $Y$.
To investigate the robustness against corrupted data, we sample training sets of 1000 points and exchange a proportion of the covariates $T$ by random values generated according to $\operatorname{Uniform}([t_\text{min}, t_\text{max}])$. Figure~\ref{fig:label-noise} shows the mean-squared error of the models trained with different methods over the proportion of random covariates in the training data. We observe that for no data corruption, all estimators perform similarly, with SMM providing a small advantage. With increasing proportion of corrupted data, SMM scales more favorably compared to the baselines. We provide more details and a hyperparameter sensitivity analysis in Appendix~\ref{app:3:details}.
\begin{figure}[t]
    \centering
    \includegraphics[width=.5\linewidth]{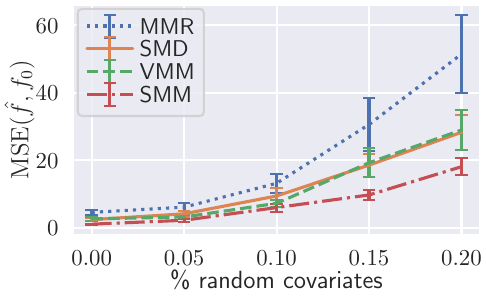}
    \caption{Robustness against corrupted data. We generate 1000 points from the process \eqref{eq:3:exp1} and substitute in a proportion of the data the treatment variable $T$ for a random value sampled uniformly over the domain. Lines and error bars correspond to the mean and standard error computed over $20$ training datasets.
    }
    \label{fig:label-noise}
\end{figure}
\paragraph{Adversarially Robust IV Regression}
We test the adversarial robustness of different IV estimators in the following setting.
Define $c = 0.2 I \in \mathbb{R}^{5\times 1}$ and the linear transformations $A \in \mathbb{R}^{1 \times 5}$, $B \in \mathbb{R}^{5 \times 1}$, with elements sampled according to $A \sim \operatorname{Uniform}([-1.5, 1.5])$ and $B \sim \operatorname{Uniform}([0.1, 0.3])$.
Consider the non-linear data generating process,
\begin{align*}
    Z &\sim \operatorname{Uniform}([-3, 3]) \\
    T &= B Z + C U + \eta_1 \\
    Y &= f_0(T) + U + \eta_2 
\end{align*}
with $U \sim N(0,1)$, $\eta_1,\eta_2 \sim N(0,0.1)$ and $f_0(t) = 1.5 \cos(At) + 0.1 At $. We approximate $f_0$ with a feed-forward neural network with $[20, 20, 3]$ hidden units and leaky ReLU activation functions. We train the network using different plug-and-play IV estimators and evaluate the adversarial robustness by running FGSM attacks~\citep{goodfellow2014explaining} in directions $\tilde{t}$ with strength $\epsilon \in [0, 1.0]$. Figure~\ref{fig:adversarial} shows that all IV estimators yield comparable mean-squared errors for $\epsilon=0$, clearly improving over the non-causal least squares (LSQ) solution (table). Moreover, for increasing attack strengths $\epsilon$, we see that SMM demonstrates stronger adversarial robustness than the SMD and VMM estimators. Interestingly, here, the MMR estimator which performed worse in the first experiments exhibits the least sensitivity towards adversarial perturbations. This might be understood by the fact that the MMR estimator corresponds to the limit case of SMM and VMM for $\lambda \rightarrow \infty$. Generally, strong regularization promotes flat functions which are less sensitive to the inputs, which might explain MMR's superior robustness here.
\begin{figure}[t]
    \centering
    \includegraphics{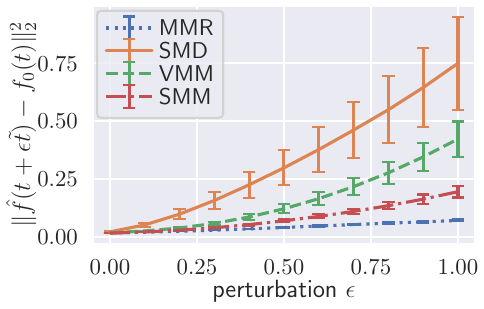}
    
    \begin{tabular}{cccccc}
      & LSQ & MMR & SMD & VMM & SMM \\
      \toprule
      MSE ($\epsilon=0$) &    
                   $0.45$ &  $0.014 $ & $ 0.018 $ & $ 0.012 $ & $0.012$ \\
    \bottomrule
    \end{tabular}
    \caption{Adversarial robustness of IV estimators. We use a training set of size $n=1000$ and evaluate the learned models over FGSM attacks with increasing strength $\epsilon$. Lines and error bars show the mean and standard error over $20$ random training datasets. The table contains the MSE in the perturbation-free case.}
    \label{fig:adversarial}
\end{figure}

In Appendix~\ref{app:3:additional} we provide results on a common modern IV benchmark that provides further evidence that SMM performs on par with state-of-the art estimators in standard IV settings. In this context, we also provide results for a Neural-SMM estimator, which proves to be competitive with state-of-the art deep learning approaches~\citep{bennett2020deep,kremer2022functional}.

\section{Related Work} \label{sec:3:related-work}
Instrumental variable regression has traditionally been addressed via the 2-stage least squares (2SLS) method, which limits both regression stages to linear models~\citep{angrist2008mostly}. Extensions to non-linear models have been provided by multiple works~\citep{AMEMIYA1974105}, recently based on density estimators~\citep{deepiv,singh2019kernel} and deep features~\citep{xu2021learning}.
As an alternative to 2SLS, estimators based on the conditional moment restriction formulation have been used based on either basis function expansions of $L^2$ \citep{ai2003efficient,Carrasco1,Carrasco2,otsu2011empirical} or machine learning models~\citep{bennett2020deep,Dikkala20:Minimax,muandet2020dual,kremer2022functional,pmlr-v202-kremer23a,bennett2020variational}.
Related to our Kernel-SMM estimator, multiple works have used RKHS functions as instrument models~\citep{Carrasco1,singh2019kernel,bennett2020variational,zhang2021maximum}, leading to similar formulations as our \eqref{eq:3:kernel-smm}. However, in contrast to ours, none of them take into account the geometry of the moment violation with respect to the data.  
Optimization over measure spaces by means of minimizing some notion of distributional distance between the optimization variable and an empirical distribution has recently attracted significant attention in the context of distributionally robust optimization~\citep{duchi2018statistics,sinha2018certifiable,MohajerinEsfahani2018,duchi2017variancebased,lamRecoveringBestStatistical2019,duchi2020learning}. On a higher level, one can distinguish between three types of approaches based on the respective distance notion (cf.\ Figure~\ref{fig:distances}): $\varphi$-divergences restrict the optimization variable to a finite dimensional vector of weights attributed to the data points and thus find optimal reweightings of the sample. Methods based on maximum-mean discrepancy~\citep{gretton2007kernel} and the Fisher-Rao metric~\citep{bauer2016uniqueness}, allow for creation and annihilation of probability mass~\citep{zhu2021kernel,pmlr-v202-kremer23a,yan2023learning}. Finally, methods based on optimal transport distances effectively allow to move around the data points in the data space~\citep{MohajerinEsfahani2018,sinha2018certifiable}.
While CMR estimation has been based on the previous two paradigms, to the best of our knowledge, our Sinkhorn Method of Moments is the first estimator based on the latter category. 
In a different context, empirical likelihood has previously been combined with Wasserstein distances to calibrate the radius of ambiguity sets in distributionally robust optimization (DRO)~\citep{blanchet2019robust}. However, their method does not extend to CMR estimation and neither does it make use of a regularized duality structure. From a mathematical perspective the derivation of our first duality result (Theorem~\ref{th:3:duality}) closely resembles the derivation of the dual Sinkhorn DRO estimator of \citet{wang2023sinkhorn}, which, nevertheless, addresses an entirely different problem. In addition, \citet{wang2023sinkhorn} relies on de-biasing techniques to optimize their objective, whereas we provided a form that can be directly optimized via stochastic gradient methods.

\section{Conclusion} \label{sec:3:conclusion}
Instrumental variable regression is an important concept in the field of causal inference, which motivates the development of estimators adapted to the intricacies of real-world datasets.
Notwithstanding recent mini-max estimators based on neural network instrument function classes showing convincing performance on benchmarks~\citep{Dikkala20:Minimax,bennett2020deep,kremer2022functional,pmlr-v202-kremer23a}, there remains a need for simple plug-and-play estimators that can be trained by practitioners without deep technical knowledge and with a manageable set of hyperparameters. We have extended the repertoire of such estimators by a method whose learning signal arises from an optimal transport geometry in the data space. We showed that our estimator exhibits favorable properties in presence of corrupted data or adversarial examples while maintaining performance competitive with state-of-the art approaches on standard benchmarks.
The simplicity of our plug-and-play estimator partially results from its kernel-based implementation which limits the scalability to large sample sizes. To address this, we provide a neural network-based implementation in the appendix, whose detailed analysis is left for future work.

\section*{Acknowledgements}
We thank Yassine Nemmour and Frederike L\"ubeck for helpful initial discussions on the project as well as Frederik Tr\"auble for insisting on using LogDet=TrLog somewhere.

\bibliography{refs}
\bibliographystyle{icml2024}

\newpage
\appendix
\onecolumn

\section{Experimental Details} \label{app:3:details}
\paragraph{Hyperparameters}
For SMM we choose the hyperparameters from the grid defined by $\epsilon \in [10^{-6},10^{-4}, 10^{-2}]$ and $\lambda/\epsilon \in [10^{-6},10^{-4}, 10^{-2}, 1.0]$. Note that as $\epsilon$ and $\gamma_t$ only appear as $\epsilon/\gamma_t$, we absorb the factor $\gamma_t$ into $\epsilon$ and consider $\gamma_t=1$ everywhere. For VMM we choose the hyperparameters from $\lambda \in [10^{-6},10^{-4}, 10^{-2}, 1.0]$ as done by the authors of the method~\citep{bennett2020variational}.
We pick the best hyperparameter configuration by evaluating the MMR objective~\citep{zhang2021maximum} on a validation data set of the same size as the training set.
We visualize the dependency on the hyperparameters for the first experiment without random covariates in Figure~\ref{fig:hparams}. We observe that the method is rather insensitive to the choice of $\epsilon$ but admits a stronger dependence on the choice of the regularization parameter $\lambda$.
\begin{figure}[t]
    \centering
    \includegraphics[width=.45\linewidth]{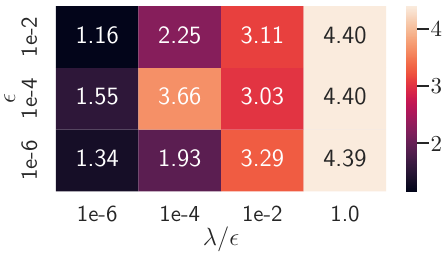}
    \caption{Kernel-SMM dependency on hyperparameters. We evaluate the SMM estimator on the first experiment without random covariates for different hyperparameter configurations. Values correspond to the mean of the prediction error $E[ \| f(T;\hat{\theta}) - f(T;\theta_0) \|_2^2]$ averaged over models trained on $20$ random training sets.}
    \label{fig:hparams}
\end{figure}

\section{Additional Results} \label{app:3:additional}
\paragraph{NetworkIV}
Here, we consider a common modern benchmark for IV regression in the standard setting without any data corruptions. 
Consider the following data generating process introduced by \citet{bennett2020deep} and subsequently used by many other works~\citep{zhang2021maximum,kremer2022functional,pmlr-v202-kremer23a},
\begin{align*}
 &y = f_0(t) + e + \delta,   &t = z + e + \gamma, \\
 &z \sim \operatorname{Uniform}([-3, 3]),& \\
   & e \sim N(0,1),  & \gamma, \delta \sim N(0, 0.1),
\end{align*}
where the function $f_0$ is chosen from the set of simple functions
\begin{center}
\begin{tabular}{ l l }
 $\operatorname{sin:} f_0(t) = \sin(t)$, & $\operatorname{abs:} f_0(t) = |t| \vspace{.5em}$,\\ 
 $\operatorname{linear:} f_0(t) = t$, & $\operatorname{step:} f_0(t) = I_{\{t \geq 0 \}}$.
\end{tabular}
\end{center}
We learn a neural network $f_\theta$ with two layers of $[20, 3]$ hidden units and leaky ReLU activation functions to approximate the function $f_0$ by imposing the conditional moment restriction $E[Y-f_\theta(T)|Z ] = 0 \ P_Z\mathrm{-a.s.}$. Table~\ref{tab:networkiv} contains the results of different plug-and-play IV estimators trained on a dataset of $1000$ points and averaged over $20$ random training datasets. We observe that SMD, VMM and SMM perform roughly on par whereas MMR only improves in one of the settings over the non-causal least squares solution (LSQ) which ignores the instruments entirely.

\begin{table}[t]
    \caption{NetworkIV experiment. Results represent the mean and standard error of the prediction error $E[ \| f(T;\hat{\theta}) - f(T;\theta_0) \|_2^2]$ resulting from 20 random training datasets.}
    \label{tab:networkiv}
    \centering
    \begin{tabular}{lccccc}
    \toprule
            & LSQ           & MMR                                       & SMD           & VMM           & SMM           \\
     \midrule
     sin    & $0.36\pm0.03$ & $0.40\pm0.02$                             & $0.12\pm0.01$ & $0.17\pm0.02$ & $0.15\pm0.01$ \\
     abs    & $1.94\pm1.48$ & $0.61\pm0.28$                             & $0.20\pm0.08$ & $0.09\pm0.04$ & $0.12\pm0.04$ \\
     step   & $0.35\pm0.04$ & $> 100$ & $0.04\pm0.01$ & $0.05\pm0.01$ & $0.04\pm0.00$ \\
     linear & $0.36\pm0.05$ & $0.36\pm0.09$                             & $0.07\pm0.04$ & $0.03\pm0.01$ & $0.07\pm0.03$ \\
    \bottomrule
    \end{tabular}
\end{table}

\paragraph{Neural Estimators}
We explore an alternative SMM implementation where we represent the instrument function $h \in \mathcal{H}$ as a neural network parameterized by $\omega \in \Omega$. With this choice, the estimator \eqref{eq:3:smm} takes the form
\begin{align}
    f^\ast = \argmin_{f \in \mathcal{F}} \max_{\omega \in \Omega} E_{\ppn}\left[\left( I+\frac{\epsilon}{2} \Delta_\xi \right) \left( \psi(\cdot;f)^T h_\omega(\cdot) \right)(\xi)  - \frac{\epsilon}{2}  \| \nabla_\xi \left( \psi(\cdot;f)^T h_\omega(\cdot) \right)(\xi) \|_{\Gamma}^2  - \frac{\lambda}{2} \| h_\omega(Z) \|_2^2 \right].
\end{align}
The Neural-SMM estimator can be trained in the same fashion as the DeepGMM~\citep{bennett2020deep} or FunctionalGEL~\citep{kremer2022functional} estimators by alternating mini-batch stochastic gradient descent steps in the the model parameters and the adversary parameters $\omega$.

We benchmark the Neural-SMM estimator against DeepGMM~\citep{bennett2020deep} and FunctionalGEL~\citep{kremer2022functional} which achieved state-of-the-art results on several benchmarks including the NetworkIV experiment.
For all methods we use the same instrument network architecture consisting of a feed-forward neural network with $[50, 20]$ hidden units and leaky ReLU activation functions. We optimize the objective by alternating steps with an optimistic Adam~\citep{daskalakis2018training} optimizer with parameters $\beta = (0.5, 0.9)$.
We tuned the learning rates, for the model and adversary by evaluating the DeepGMM estimator for different values and fix them both to $5e^{-4}$ for all methods. In the same way we fix the batch size to $200$ and the number of epochs to $3000$. For the FunctionalGEL estimator we use the Kullback-Leibler divergence version. For all methods we choose the regularization parameter $\lambda$ from $[10^{-6},10^{-4}, 10^{-2}, 1.0]$ and for Neural-SMM we additionally choose $\epsilon$ from $[10^{-6},10^{-4}, 10^{-2}, 1.0]$ by using the MMR objective on a validation set of the same size as the training set.

We observe in Table~\ref{tab:3:neuralnetworkiv} that Neural-SMM performs on par with these SOTA estimators on all variants of the NetworkIV experiment, suggesting that the geometry-awareness and additional robustness of our estimator does not come at the price of reduced performance in standard settings. It does, however, come at the price of increased computation due to the presence of the gradient and Laplace operators with respect to the data in the objective.

Figure~\ref{fig:3:neuralhparam} visualizes the dependence of Neural-SMM on its hyperparameters. We observe that for this experiment SMM requires either one or both parameters to be chosen large for optimal performance but the performance remains stable across a range of parameters.

\begin{table}[t]
    \caption{Neural CMR estimators. Results represent the mean and standard error of the prediction error $E[ \| f(T;\hat{\theta}) - f(T;\theta_0) \|_2^2]$ resulting from 20 random runs of the NetworkIV experiment.}
    \label{tab:3:neuralnetworkiv}
    \centering
    \begin{tabular}{lccc}
    \toprule
            & DeepGMM       & NeuralFGEL          & NeuralSMM     \\
    \midrule
     sin    & $0.08\pm0.01$ & $0.10\pm0.01$ & $0.07\pm0.01$ \\
     abs    & $0.04\pm0.01$ & $0.04\pm0.01$ & $0.04\pm0.01$ \\
     step   & $0.07\pm0.01$ & $0.08\pm0.01$ & $0.07\pm0.01$ \\
     linear & $0.05\pm0.01$ & $0.06\pm0.01$ & $0.05\pm0.01$ \\
    \bottomrule
    \end{tabular}
\end{table}

\begin{figure}[t]
    \centering
    \includegraphics[width=.45\linewidth]{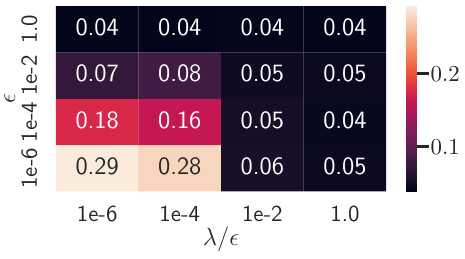}
    \caption{Neural-SMM dependency on hyperparameters. We evaluate the Neural-SMM estimator for different hyperparameter configurations exemplarily for the $\mathrm{abs}$ function in the network IV experiment. Values correspond to the mean of the prediction error $E[ \| f(T;\hat{\theta}) - f(T;\theta_0) \|_2^2]$ averaged over models trained on $20$ random training sets.}
    \label{fig:3:neuralhparam}
\end{figure}

\newpage

\section{Proofs}
\subsection{Duality Results}
\paragraph{Proof of Theorem~\ref{th:3:duality}}
\begin{proof}
    Introducing the Lagrange parameter $\rho \in \mathbb{R}$, the Lagrangian of \eqref{eq:3:primal} reads
    \begin{align}
        L(P,\rho,f) =& \min_{\pi \in \Pi(P, \ppn)} E_{(\xi,\xi') \sim \pi} \left[c(\xi,\xi') + \epsilon \log \left( \frac{\dd \pi(\xi,\xi')}{\dd \mu(\xi) \dd \nu(\xi')} \right) \right] \\ & + \rho \sup_{h \in \mathcal{H}, \| h\|_\mathcal{H} =1}E_P[\Psi(\xi;f)(h)].
    \end{align}
    As eventually the Lagrangian will be maximized with respect to $\rho$, we can merge it with the optimization over the unit ball in $\mathcal{H}$ to obtain a Lagrangian with an unrestricted parameter $h \in \mathcal{H}$,
    \begin{align}
        L(P,h,f) = \min_{\pi \in \Pi(P, \ppn)} E_{(\xi,\xi') \sim \pi} \left[c(\xi,\xi') + \epsilon \log \left( \frac{\dd \pi(\xi,\xi')}{\dd \mu(\xi) \dd \nu(\xi')} \right) \right] + E_P[\Psi(\xi;f)(h)].
    \end{align}
    Note that the Wasserstein distance is mass preserving, i.e., we do not need to explicitly impose the constraint $\ee_\pp[1]= 1$ as this is implied directly by normalization of the empirical distribution, i.e., let $p$ and $\hat{p}$ denote the density and probability mass functions of $\pp$ and $\ppn$ respectively, then $\ee_\pp[1] = \int_\Xi p(\xi) \dd \xi = \int_\Xi \sum_{i=1}^n \pi(\xi,\xi'_i)  \dd \xi = \sum_{i=1}^n \hat{p}(\xi'_i) = \sum_{i=1}^{n} \frac{1}{n}= 1$. 
    
    To derive the dual problem we need to minimize the Lagrangian over the primal variable $P$.
    By definition of the coupling distribution $\pi$ we have $p = \mathbb{P}_{1\sharp} \pi$ and thus we can collapse the minimizations over the $\pi$ and $P$ into a single minimization over $\pi \in \Pi(\ppn) := \{\mathcal{P}(\Xi \times \Xi) : \mathbb{P}_{2\sharp} \pi = \ppn \}$,
    \begin{align}
        D(h,f) = \min_{\pi \in \Pi(\ppn)} E_{(\xi,\xi') \sim \pi}\left[c(\xi,\xi') + \epsilon \log \left( \frac{\dd \pi(\xi,\xi')}{\dd \mu(\xi) \dd \nu(\xi')} \right) \right] + E_{\mathbb{P}_{1\sharp}\pi}[\Psi(\xi;f)(h)]. \label{eq:3:dual1}
    \end{align}
    Now to extract the relevant degree of freedom we can write all expectation operators as combinations of the empirical expectation and conditional expectation over $\pi(\xi,\xi')$ given its second argument $\xi' \in \Xi$. 
    To see this, note that by the product rule we have $\pi(\xi, \xi') =: \pi(\xi| \xi') \hat{p}(\xi')$ and by the law of iterated expectation we have for any function $g : \Xi \times \Xi \rightarrow \mathbb{R}$, $E_\pi[g(\xi, \xi')] = E_{\xi'\sim \ppn}[E_{\xi\sim \pi|\xi'}[g(\xi,\xi')|\xi']]$, where we defined $\pi|\xi'$ as the conditional distribution of $\xi$ given $\xi'$, with density $\pi(\xi|\xi')$. Similarly we have for any function $g : \Xi \rightarrow \mathbb{R}$,
    \begin{align}
        E_{P_{1\sharp}\pi}[g(\xi)] &= \int_\Xi g(\xi) (\mathbb{P}_{1\sharp}\pi)(\xi) \dd \xi = \int_{\Xi} g(\xi) \sum_{i=1}^n \pi(\xi, \xi_i') \dd \xi \\
        &= \int_\Xi g(\xi) \sum_{i=1}^n \pi(\xi | \xi_i') \hat{p}(\xi_i') \dd \xi = \int_\Xi g(\xi) \frac{1}{n} \sum_{i=1}^n \pi(\xi | \xi_i') \dd \xi \\
        &= E_{\xi' \sim \ppn}[E_{\xi \sim \pi|\xi'}[g(\xi) | \xi']].
    \end{align}
    Therefore the optimization over $\pi \in \Pi(\ppn)$ is equivalent to a sequence of optimization problems over $\pi|\xi' \in \mathcal{P}(\Xi)$, one for each value of $\xi'\in \Xi$.
    With this we can express the dual problem \eqref{eq:3:dual1} as
    \begin{align}
        D(h,f) = E_{\xi' \sim \ppn} \left[ \min_{\pi|\xi' \in \mathcal{P}(\Xi)}  E_{\xi \sim \pi|\xi'}\left[ c(\xi,\xi') + \epsilon \log \left( \frac{\dd (\pi|\xi')(\xi)}{\dd \mu(\xi)}\right) + \Psi(\xi;f)(h) \Bigg| \xi' \right]\right] \label{eq:3:dual3}
    \end{align}
    Now for each $\xi' \in \Xi$ consider the inner optimization problem
    \begin{align}
        G(\xi';h,f) := \min_{\pi|\xi' \in \mathcal{P}(\Xi)}  E_{\xi \sim \pi|\xi'}\left[ c(\xi,\xi') + \epsilon \log \left( \frac{\dd (\pi|\xi')(\xi)}{\dd \mu(\xi)}\right) + \Psi(\xi;f)(h) \right].
    \end{align}
    Define the density of $\pi|\xi' \in \mathcal{P}(\Xi)$ with respect to the reference measure $\mu \in \mathcal{P}(\Xi)$ as $r(\xi) = \frac{\dd (\pi|\xi')(\xi)}{\dd \mu(\xi)} $, then we can rewrite the optimization problem as an optimization over $r \in \mathcal{R} := \{r : \Xi \rightarrow \mathbb{R}_+:  E_\mu[r(\xi)] = 1 \}$,
    \begin{align}
        G(\xi';h,f) = \min_{r \in \mathcal{R}} E_{\xi \sim \mu}\left[r(\xi) c(\xi,\xi') + \epsilon r(\xi) \log\left(r(\xi) \right) + r(\xi) \Psi(\xi;f)(h) \right].
    \end{align}
    Now introducing Lagrange parameter $\eta \in \mathbb{R}$ and using Lagrangian duality we get
    \begin{align}
        G(\xi';h,f) &= \sup_{\eta \in \mathbb{R}} \min_{r : \Xi \rightarrow \mathbb{R}_+} E_{\xi \sim \mu}[ r(\xi) c(\xi,\xi') + \epsilon r(\xi) \log\left(r(\xi) \right) + r(\xi) \Psi(\xi;f)(h) + \eta (1 - r(\xi) )] \\
        &= \sup_{\eta \in \mathbb{R}} \eta - \epsilon E_{\xi \sim \mu}\left[ \sup_{t \geq 0} t \frac{\eta - c(\xi,\xi') - \Psi(\xi;f)(h)}{\epsilon} - t \log t  \right] \\
       &= \sup_{\eta \in \mathbb{R}} \eta - \epsilon E_{\xi \sim \mu}\left[ \exp\left( \frac{\eta - c(\xi,\xi') - \Psi(\xi;f)(h)}{\epsilon} - 1 \right) \right], \label{eq:3:dual2}
    \end{align}
    where we used that the Fenchel conjugate of the Kullback Leibler divergence $t \log t$ is $\sup_t \langle p, t \rangle - t \log t = e^{p-1}$.
    We can eliminate the dual normalization variable $\eta \in \mathbb{R}$ from the problem by solving the corresponding first order optimality condition
    \begin{align}
        0 = 1 - e^{\eta/\epsilon - 1}  \ee_{X\sim \mu}\left[
     \exp\left(\frac{- \Psi(\xi;f)(h) -c(\xi,\xi')}{\epsilon} \right)
      \right],
    \end{align}
    which yields
    \begin{align}
        \eta = \epsilon - \epsilon \log  \ee_{X\sim \mu}\left[ \exp\left(\frac{- \Psi(\xi;f)(h) -c(\xi,\xi')}{\epsilon} \right)
      \right].
    \end{align}
    Inserting back into \eqref{eq:3:dual2}, we obtain for each $\xi' \in \Xi$
    \begin{align}
        G(\xi';h,f) = - \epsilon \log E_{\xi \sim \mu}\left[ \exp\left(\frac{- \Psi(\xi;f)(h) -c(\xi,\xi')}{\epsilon} \right) \right].
    \end{align}
    and the result follows by inserting into \eqref{eq:3:dual3} and redefining $h/\epsilon \rightarrow h$.
    \end{proof}

\paragraph{Proof of Theorem~\ref{th:3:linearization}}
\begin{proof}
    Using the assumptions on the reference measure and cost function, we can write the objective in the form \eqref{eq:3:linear1}, where the inner expectation is given as
    \begin{align}
        E_{\xi \sim \mathcal{N}(\xi', \epsilon \Gamma^{-1})} \left[ e^{-\Psi(\xi;f)(h)} \right] 
        &= \int_{\Xi} e^{-\Psi(\xi;f)} e^{-\frac{1}{2\epsilon} \|\xi -\xi' \|_{\Gamma^{-1}}^2} \dd \xi.
    \end{align}
    As for small $\epsilon$ the integrand only provides a finite contribution in a neighborhood of $\xi'$, we can use that $\Psi$ is continuously differentiable everywhere and employ a Taylor expansion,
    \begin{align}
        \Psi(\xi;f)(h) =& \Psi(\xi';f)(h) + (\xi - \xi')^T \nabla_\xi \Psi(\xi';f)(h) \\
        &+ \frac{1}{2}(\xi - \xi')^T \nabla^2_\xi \Psi(\xi';f)(h) (\xi - \xi') + O(\|\xi - \xi' \|^3).
    \end{align}
    Note that due to the Gaussian measure under the integral we have $\|\xi - \xi' \| = O(\epsilon^{1/2})$.
    Now defining $\delta := \xi - \xi' \in \Xi$ as well as the gradient $G(\xi') := \nabla_\xi \Psi(\xi';f)(h)$ and Hessian $H(\xi') := \nabla^2_\xi \Psi(\xi';f)(h)$ of the evaluated moment functional we can insert back and get
    \begin{align}
        E_{\xi \sim \mathcal{N}(\xi', \gamma)} \left[ e^{-\Psi(\xi;f)(h)} \right] = e^{-\Psi(\xi';f)(h)} \int_\Xi \exp \left( - \frac{1}{2\epsilon} \left( 2 \epsilon \delta^T G(\xi') + \epsilon \delta^T H(\xi') \delta + \delta^T \Gamma \delta \right) \right) \dd \delta + O(\epsilon^{3/2}).
    \end{align}
    Define the regularized Hessian $\Omega_\epsilon := \Omega_\epsilon(\xi') := \Gamma + \epsilon H(\xi')$, which is invertible w.p.1, as for sufficiently small $\epsilon/\gamma$ we have $\lambda_\mathrm{min}(\Gamma) = \min_{w \in \{t,y,z \}}\gamma_w > \epsilon \lambda_{\mathrm{min}}(H(\xi'))$ w.p.1 and thus $\Omega_\epsilon$ is strictly positive definite w.p.1. Then we can employ a change of variables by defining $\omega := \Omega_\epsilon^{1/2} \delta$ and obtain
    \begin{align}
         &E_{\xi \sim \mathcal{N}(\xi', \gamma)} \left[ e^{- \Psi(\xi;f)(h)} \right] \\
         =& e^{- \Psi(\xi';f)(h)} \int \frac{1}{\left| \det \Omega_\epsilon^{1/2} \right|} \times \exp\left( -\frac{1}{2\epsilon} \left(\omega^T \omega + 2\epsilon \omega^T \Omega_{\epsilon}^{-1/2} G(\xi') 
         \right) \right) \dd \omega + O(\epsilon^{3/2}).
    \end{align}
    Now, completing the square we obtain
    \begin{align}
        &E_{\xi \sim \mathcal{N}(\xi', \gamma)} \left[ e^{-\Psi(\xi;f)(h)} \right] \\
        =& e^{-\Psi(\xi';f)(h)} e^{\frac{\epsilon}{2} G(\xi')^T \Omega_\epsilon^{-1} G(\xi')} 
        \int \frac{1}{\left| \det \Omega_\epsilon^{1/2} \right|} \exp \left(- \frac{1}{2\epsilon} \left(\omega + \epsilon \Omega_\epsilon^{-1/2} G(\xi')^2 \right) \right) \dd \omega + O(\epsilon^{3/2})\\
        =& \left(\frac{2\pi}{\epsilon}\right)^{d_{\xi}/2}  \left| \det \Omega_\epsilon^{1/2} \right|^{-1}  e^{-\Psi(\xi';f)(h)} e^{\frac{\epsilon}{2} G(\xi')^T \Omega_\epsilon^{-1} G(\xi')} + O(\epsilon^{3/2}).
    \end{align}
    Finally inserting back into \eqref{eq:3:linear1} we get
    \begin{align}
        D(f,h) =& E_{\xi'\sim \ppn}\left[ - \epsilon \log E_{\xi \sim \mathcal{N}(\xi', \gamma)} \left[ e^{\Psi(\xi;f)(h)} \right] \right] \\
        =& E_{\xi'\sim \ppn}\left[ \epsilon \Psi(\xi';f)(h) - \frac{\epsilon^2}{2} G(\xi')^T \Omega_\epsilon^{-1} G(\xi') + \frac{\epsilon}{2} \log  \left| \det \Omega_\epsilon \right| \right] - \frac{\epsilon d_\xi}{2} \log \frac{2\pi}{\epsilon} + O(\epsilon^{5/2}).
    \end{align}
    Dividing by $\epsilon$ and neglecting constant terms we get
    \begin{align}
        D(f,h) = E_{\xi'\sim \ppn}\left[ \Psi(\xi';f)(h) - \frac{\epsilon}{2} G(\xi')^T \Omega_\epsilon^{-1} G(\xi') + \frac{1}{2} \log  \left| \det \Omega_\epsilon \right| \right] + O(\epsilon^{3/2}).
    \end{align}
    Now, for small $\epsilon$ we can Taylor expand $\Omega_\epsilon^{-1}$ as
    \begin{align}
        \Omega_\epsilon^{-1} &= \left(\Gamma + \epsilon H(\xi') \right)^{-1} \\
        &= \Gamma^{-1} \left(I + \epsilon \Gamma^{-1} H \right)^{-1} \\
        &= \Gamma^{-1} \left( I - \epsilon \Gamma^{-1} H \right)  + O(\epsilon^2) \\
        &= \Gamma^{-1} + O(\epsilon).
    \end{align}
    Similarly we have 
    \begin{align}
        \log \left|\det \Omega_\epsilon \right| &= \log \left|\det \left( \Gamma + \epsilon H \right)\right| \\
        &= \log \left| \det \Gamma \right| + \log \left| \det \left(I + \epsilon \Gamma^{-1} H \right) \right|  \\
        &= \underbrace{\left(\sum_{x \in \{t,y,z \}} d_x \log \gamma_x \right)}_{=:C} + \log \det \left( I + \epsilon \Gamma^{-1} H\right) \\
        &= C  + \operatorname{Tr} \log \left(I + \epsilon \Gamma^{-1} H \right) \\
        &= C + \operatorname{Tr} \left( \epsilon \Gamma^{-1} H + O(\epsilon^2) \right) \\
        &= C + \epsilon \sum_{x \in \{t, y, z\}} \frac{1}{\gamma_x} \Delta_x \Psi(\xi';f)(h) + O(\epsilon^2).
    \end{align}
    So we finally obtain
    \begin{align}
        D(f,h) =& E_{\ppn}\left[\Psi(\xi;f)(h)  - \frac{\epsilon}{2} \sum_{x \in \{t, y, z\}} \frac{1}{\gamma_x} \left( \| \nabla_x \Psi(\xi;f)(h) \|_2^2 -  \Delta_x \Psi(\xi;f)(h) \right) \right] + O(\epsilon^{3/2}).
    \end{align}
\end{proof}

\subsection{Proof of Theorem~\ref{th:3:consistency} (Consistency)}
The objective of the SMM estimator \eqref{eq:3:smm} can be written as
\begin{align}
    \widehat{D}(h, \theta) = \left(I + \frac{\epsilon}{2} \Delta_\xi \right) E_{\ppn}[\Psi(\xi;f)(h)] - \frac{\epsilon}{2} \langle h, \widehat{\Omega}_{\lambda_n}(\bar{\theta}_n) h \rangle_\mathcal{H} ,
\end{align}
where we defined the linear operator $\widehat{\Omega}_{\lambda_n}(\bar{\theta}_n): \mathcal{H} \rightarrow \mathcal{H}$ as $\widehat{\Omega}_{\lambda_n}(\bar{\theta}_n) = E_{\ppn} \left[ \left( \nabla_\xi \Psi(\xi;\bar{\theta}_n) \right)^T\Gamma^{-1} \nabla_\xi \Psi(\xi;\bar{\theta}_n) \right] + \lambda_n I \otimes I$. Our proof of Theorem~\ref{th:3:consistency} uses properties of the spectrum of $\widehat{\Omega}_{\lambda_n}(\bar{\theta}_n)$ which we will derive in the following.

\subsubsection{Previous Results}
\begin{lemma}[Corollary~9.31, \citet{kosorok2008introduction}] \label{lemma:3:donsker}
    Let $\mathcal{F}$ and $\mathcal{G}$ be Donsker classes of functions. Then $\mathcal{F} + \mathcal{G}$ is Donsker. Further if additionally $\mathcal{F}$ and $\mathcal{G}$ are uniformly bounded, then $\mathcal{F} \cdot \mathcal{G}$ is Donsker.
\end{lemma}

\begin{lemma}[Lemma~18, \citet{bennett2020variational}] \label{lemma:3:bennett}
Suppose that $\mathcal{G}$ is a class of functions of the form $g : \Xi \rightarrow \mathbb{R}$, and that $\mathcal{G}$ is $P$-Donsker in the sense of \citet{kosorok2008introduction}. Then we have
\begin{align}
    \sup_{g \in \mathcal{G}} E_{\ppn}[g(\xi)] - E[g(\xi)] = O_p(n^{-1/2}).
\end{align}
\end{lemma}

\begin{lemma}[Lemma~E.4, \citet{pmlr-v202-kremer23a}] \label{lemma:3:kremer}
    Let Assumptions~\ref{as2:psi}-\ref{as10:smooth-param} be satisfied. Then the matrix
    \begin{align}
        \Sigma(\theta_0) = \left\langle  E[ \nabla_\theta \Psi(\xi;\theta_0)], E[ \nabla_{\theta^T} \Psi(\xi;\theta_0)] \right\rangle_{\mathcal{H}^\ast}
    \end{align}    
    is strictly positive definite and non-singular with smallest eigenvalue bounded away from zero.
\end{lemma}

\subsubsection{Spectrum of $\widehat{\Omega}$}

\begin{lemma} \label{lemma:3:boundedness}
    Let Assumptions~\ref{as9:data} and \ref{as4:regularity} be satisfied. Then we have
    \begin{align}
        & \sup_{\theta \in \Theta, x \in \mathcal{T} \times \mathcal{Y}} \|\psi(x;\theta)\|_\infty \leq C_\psi < \infty \\
        & \sup_{\theta \in \Theta, x \in \mathcal{T} \times \mathcal{Y}}  \|J_x(\psi)(x;\theta)\|_\infty \leq L_\psi < \infty \\
        & \sup_{\theta \in \Theta, x \in \mathcal{T} \times \mathcal{Y}}  \|\Delta_x \psi(x;\theta)\|_\infty \leq D_\psi < \infty \\
        & \sup_{\theta \in \Theta, z \in \mathcal{Z}}  \| h(z)\|_\infty \leq C_h < \infty \\
        & \sup_{\theta \in \Theta, z \in \mathcal{Z}}  \|J_z h(z)\|_\infty \leq L_h < \infty \\
        & \sup_{\theta \in \Theta, z \in \mathcal{Z}}  \| \Delta_z h(z)\|_\infty \leq D_h < \infty,
    \end{align}
    which directly implies $\| \Delta_\xi \|_\mathrm{op} < \infty$ on $\mathcal{H}^\ast$.
\end{lemma}
\begin{proof}
    The proof follows directly from the fact that a continuous function on a compact domain is bounded and both $\psi(\cdot;\theta)$ and $h$ are $C^\infty$-smooth by Assumptions~\ref{as4:regularity} and $\ref{as1:instrument}$.
\end{proof}

\begin{lemma}\label{lemma:3:cond-covariance}
    Let $V(Z;\theta) = E[J_x(\psi)(X;\theta) \Gamma^{-1} J_x(\psi)(X;\theta)^T |Z]$ be non-singular with probability $1$. Then the linear operator $\Omega(\theta): \mathcal{H} \rightarrow \mathcal{H}$ defined as
    \begin{align}
        \Omega(\theta) = E \left[ \left( \nabla_\xi \Psi(\xi;\theta) \right)^T\Gamma^{-1} \nabla_\xi \Psi(\xi;\theta) \right] 
    \end{align}
    is non-singular.
\end{lemma}
\begin{proof}
    We derive the result by showing that the smallest eigenvalue of $\Omega(\theta)$ is positive. Consider any $h \in \mathcal{H}$ with $\| h \|_{L^2(\mathcal{H}, P_0)} >0$ then we have
    \begin{align}
        \langle h, \Omega(\theta) h \rangle_\mathcal{H} =& E[h(Z)^T J_x(\psi)(X;\theta) \Gamma^{-1} J_x(\psi)(X;\theta) h(Z) ] \\
        =& E[h(Z)^T E[ J_x(\psi)(X;\theta) \Gamma^{-1} J_x(\psi)(X;\theta) | Z] h(Z) ] \\
        =& E[h(Z)^T V_0(Z;\theta) h(Z)] \\
        =& C E[\|h(Z)\|_2^2] \\
        =& C \|h \|_{L^{2}(\mathcal{H},P_0)}^2 > 0
    \end{align}
    where we used that by assumption $V(Z;\theta)$ is non-singular and thus its smallest eigenvalue $C$ bounded away from zero w.p.1.  
\end{proof}

\begin{lemma}[Spectrum of $\widehat{\Omega}$] \label{lemma:3:covariance}
    Let the assumptions of Theorem~\ref{th:3:consistency} be satisfied.
    Then for $\bar{\theta} \in \Theta$ with $\bar{\theta}_n \rightarrow \bar{\theta}$, the empirical gradient covariance operator 
    \begin{align}
        \widehat{\Omega}_{\lambda_n}(\bar{\theta}_n) = E_{\ppn} \left[ \left( \nabla_\xi \Psi(\xi;\bar{\theta}_n) \right)^T\Gamma^{-1} \nabla_\xi \Psi(\xi;\bar{\theta}_n) \right] + \lambda_n I \otimes I
    \end{align}
    is a positive definite operator with smallest eigenvalue $\lambda_\mathrm{min}(\widehat{\Omega})$ bounded away from zero and largest eigenvalue $\lambda_\mathrm{max}(\widehat{\Omega}) < C < \infty$ bounded from above w.p.a.1.
\end{lemma}
\begin{proof}
    Let in the following $\widehat{\Omega}(\theta) =  \widehat{\Omega}_{\lambda_n=0}(\theta)$.
    With Assumption~\ref{as7:covariance} it follows from Lemma~\ref{lemma:3:cond-covariance} that the operator $\Omega(\bar{\theta}):= E\left[ \left( \nabla_\xi \Psi(\xi;\bar{\theta}) \right)^T\Gamma^{-1} \nabla_\xi \Psi(\xi;\bar{\theta}) \right]$ is non-singular and thus its smallest eigenvalue bounded away from zero.
    In the following we show that $\widehat{\Omega}(\bar{\theta}_n) \overset{p}{\rightarrow} \Omega(\bar{\theta})$, where the convergence rate in operator norm is $O_p(n^{-\zeta})$. Therefore, by adding the identity operator with regularization parameter $\lambda_n$ that goes to zero slower than $O_p(n^{-\zeta})$ we ensure that $\widehat{\Omega}_{\lambda_n}(\bar{\theta}_n)$ remains positive definite w.p.a.1.
    The derivation of this result follows the proof of Lemma~20 of \citet{bennett2020variational}.
    By the triangle inequality we have
    \begin{align}
        \|\widehat{\Omega}(\bar{\theta}_n) - \Omega(\bar{\theta}) \|_{\mathrm{op}} &\leq \| \widehat{\Omega}(\bar{\theta}_n) - \Omega(\bar{\theta}_n) \| + \| \Omega(\bar{\theta}_n) - \Omega(\bar{\theta}) \|. \label{eq:3:triangle}
    \end{align}
    
    The first term we can estimate using standard results from empirical process theory. Define $\|h \|_\mathcal{H}^2 = \frac{1}{m} \sum_{i=1}^m \| h_i \|_{\mathcal{H}_i}^2$ as well as $\jpsi(X;\theta)= J_x \psi(X;\theta)$ and $\jh(Z) = J_z h(Z)$. Let $\mathcal{H}_1 = \{h \in \mathcal{H} : \| h\|_\mathcal{H} \leq 1 \}$ denote the unit ball in $\mathcal{H}$, then
    \begin{align}
        \| \widehat{\Omega}(\bar{\theta}_n) - \Omega(\bar{\theta}_n) \| =& \sup_{h,h' \in \mathcal{H}_1 } \langle h',  \widehat{\Omega}(\bar{\theta}_n) - \Omega(\bar{\theta}_n) h \rangle_\mathcal{H} \\
        =& \sup_{h,h'\in \mathcal{H}_1 } \Bigg\{ E_{\ppn}\left[h(Z)^T \jpsi(X;\bar{\theta}_n) \Gamma^{-1}_x \jpsi(X;\bar{\theta}_n)^T h'(Z)\right] \\
        &- E\left[h(Z)^T \jpsi(X;\bar{\theta}_n) \Gamma^{-1}_x \jpsi(X;\bar{\theta}_n)^T h'(Z)\right]  \\
        &+ \frac{1}{\gamma_z} E_{\ppn}\left[\psi(X;\bar{\theta}_n)^T J_{h}(Z) J_{h'}(Z)^T \psi(X;\bar{\theta}_n)\right] \\
        &- \frac{1}{\gamma_z} E\left[ \psi(X;\bar{\theta}_n)^T J_{h}(Z) J_{h'}(Z)^T \psi(X;\bar{\theta}_n) \right] \Bigg\} \\
        \leq & \sup_{g \in \mathcal{G}^2} \left\{ E_{\ppn}[g(\xi)] - E[g(\xi)] \right\} + \frac{1}{\gamma_z} \sup_{s \in \mathcal{S}^2} \left\{ E_{\ppn}[s(\xi)] - E[s(\xi)] \right\}
        \label{eq:3:donsker}
    \end{align}
    where for $i \in [d_\xi]$ we define
    \begin{align}
        \mathcal{G}_i &= \{g_i : g_i(\xi) = \sum_{j=1}^m h_j(z) \left(\jpsi(x;\theta)\right)_{ji} \Gamma^{-1/2}_{ii}, \ h \in \mathcal{H}_{i,1},\  \theta \in \Theta \} \\
        \mathcal{G}^2 &= \{g : g(\xi) = \sum_{i \in [d_x]} g_i(\xi) g'_i(\xi), \ g_i, g'_i \in \mathcal{G}_i \} \\ 
        \mathcal{S}_i &= \{s_i: s_i(\xi) = \sum_{j=1}^m \psi_j(x;\theta) \left( J_h(z) \right)_{ji}, \ h \in \mathcal{H}_{i,1}, \ \theta \in \Theta \} \\
        \mathcal{S}^2 &= \{s_i: s_i(\xi) = \sum_{i \in [d_z]} s_i(\xi) s_i'(\xi), \ s_i, s'_i \in \mathcal{S}_i \}
    \end{align}
    Now for the first term, we have that each $h_j \in \mathcal{H}_{i,1}$ is $P_0$-Donsker by Assumption~\ref{as1:instrument} and uniformly bounded by Lemma~\ref{lemma:3:boundedness}. 
    Similarly each entry of the Jacobian $\jpsi(\cdot;\theta)$ is $P_0$-Donsker by Assumption~\ref{as4:regularity} and uniformly bounded by Lemma~\ref{lemma:3:boundedness}. With that we can employ Lemma~\ref{lemma:3:donsker} to conclude that $\mathcal{G}_i$ is $P_0$-Donsker and thus using Lemma~\ref{lemma:3:donsker} again it follows that $\mathcal{G}^2$ is $P_0$-Donsker.
    Therefore we can use Lemma~\ref{lemma:3:bennett} to obtain $\sup_{g \in \mathcal{G}^2} \left\{ E_{\ppn}[g(\xi)] - E[g(\xi)] \right\} = O_p(n^{-1/2})$.

    For the second term in \eqref{eq:3:donsker} we have that each $\psi_j(\cdot;\theta)$ is $P_0$-Donsker by Assumption~\ref{as4:regularity} and uniformly bounded by Lemma~\ref{lemma:3:boundedness}. 
    Similarly each entry of the Jacobian $J_z h$ is $P_0$-Donsker by Assumption~\ref{as1:instrument} and uniformly bounded by Lemma~\ref{lemma:3:boundedness}. With that, again, we can employ Lemma~\ref{lemma:3:donsker} to conclude that $\mathcal{S}_i$ is $P_0$-Donsker and thus using Lemma~\ref{lemma:3:donsker} again it follows that $\mathcal{S}^2$ is $P_0$-Donsker.
    Therefore we can use Lemma~\ref{lemma:3:bennett} to obtain $\frac{1}{\gamma_z} \sup_{s \in \mathcal{S}^2} \left\{ E_{\ppn}[s(\xi)] - E[s(\xi)] \right\} = O_p(n^{-1/2})$.

    Putting these results together we finally obtain $\| \widehat{\Omega}(\bar{\theta}_n) - \Omega(\bar{\theta}_n) \| \leq O_p(n^{-1/2})$.

    For the second term in \eqref{eq:3:triangle} we have 
        \begin{align}
        \| \Omega(\bar{\theta}_n) - \Omega(\bar{\theta}) \| =& \sup_{h, h' \in \mathcal{H}_1} \langle h', \Omega(\bar{\theta}_n) - \Omega(\bar{\theta}) h \rangle_\mathcal{H} \leq C_x + \frac{1}{\gamma_z} C_z \label{eq:3:bla}
        \end{align}
    where
    \begin{align}
        C_x =& \sup_{h, h' \in \mathcal{H}_1}E\Big[ h'(Z)^T \Big(\jpsi(X;\bar{\theta}_n) \Gamma^{-1}_x \jpsi(X;\bar{\theta}_n)^T \\
        &- \jpsi(X;\bar{\theta}) \Gamma^{-1}_x \jpsi(X;\bar{\theta})^T \Big) h(Z) \Big] \\
        =& \sup_{h, h' \in \mathcal{H}_1} E\Big[h'(Z)^T \Big( \jpsi(X;\bar{\theta}_n) \Gamma^{-1}_x \jpsi(X;\bar{\theta}_n) - \jpsi(X;\bar{\theta}) \Gamma^{-1}_x \jpsi(X;\bar{\theta}) \Big) h(Z)  \Big] \\
        = & \sup_{h, h' \in H_{1} } E\Big[h'(Z)^T  \jpsi(X;\bar{\theta}_n) \Gamma^{-1}_x \left( \jpsi(X;\bar{\theta}_n) - \jpsi(X;\bar{\theta})\right)^T h(Z) \\
        & +  h'(Z)^T \jpsi(X;\bar{\theta}) \Gamma^{-1}_x \left( \jpsi(X;\bar{\theta}_n) - \jpsi(X;\bar{\theta})\right)^T  h(Z) \Big] \\
        \leq & \frac{2}{\min(\{ \gamma_t, \gamma_y \})} m^2 C_h^2 L_\psi E\left[ \|\jpsi(X;\bar{\theta}_n) - \jpsi(X;\bar{\theta}) \|_\infty \right] \\
        =& O_p(n^{-\zeta})
    \end{align}
    where we used that by Lemma~\ref{lemma:3:boundedness}, $\sup_{\theta \in \Theta, x \in \mathcal{T} \times \mathcal{Y}} \| \jpsi(x;\theta) \|_\infty \leq L_\psi$ and $\sup_{h \in \mathcal{H}_1, z \in \mathcal{Z}} |h(z)| \leq C_h $ as well as by Assumption~\ref{as6:first-estimate}  $E\left[ \|\jpsi(X;\bar{\theta}_n) - \jpsi(X;\bar{\theta})  \|_\infty \right] = O_p(n^{-\zeta})$. 

    Now similarly for the second term in \eqref{eq:3:bla} we have
    \begin{align}
        C_z =& \sup_{h, h' \in \mathcal{H}_1}E\Big[ \operatorname{Tr}\left(J_{h'}(Z)^T \psi(X;\bar{\theta}_n) \psi(X;\bar{\theta}_n)^T \jh(Z) \right) \\
        & - \operatorname{Tr}\left(J_{h'}(Z)^T \psi(X;\bar{\theta}) \psi(X;\bar{\theta})^T \jh(Z) \right)\Big] \\
        \leq & L_h^2 E[\pmb{1}^T \left( \psi(X;\bar{\theta}_n) \psi(X;\bar{\theta}_n)^T - \psi(X;\bar{\theta}) \psi(X;\bar{\theta})^T \right) \pmb{1}] \\
        = & L_h^2 E\left[\pmb{1}^T \psi(X;\bar{\theta}_n) \left( \psi(X;\bar{\theta}_n) - \psi(X;\bar{\theta}) \right)^T \pmb{1}\right] \\
        & + L_h^2 E\left[\pmb{1}^T \psi(X;\bar{\theta}) \left( \psi(X;\bar{\theta}_n) - \psi(X;\bar{\theta}) \right)^T \pmb{1}\right] \\
        \leq & 2 m^2 L_{h}^2 C_\psi E\left[\| \psi(X;\bar{\theta}_n) - \psi(X;\bar{\theta}) \|_\infty \right] \\
        =& O_p(n^{-\zeta}),
    \end{align}
    where again we used Lemma~\ref{lemma:3:boundedness} and Assumption~\ref{as6:first-estimate}.
    Combining both results we obtain $\| \Omega(\bar{\theta}_n) - \Omega(\bar{\theta}) \| = \leq C_x + \frac{1}{\gamma_z} C_z \leq O_p(n^{-\zeta})$.
    
    Finally as $0 < \zeta \leq 1/2$ it follows that
    \begin{align}
        \|\widehat{\Omega}(\bar{\theta}_n) - \Omega(\bar{\theta}) \| &\leq \| \widehat{\Omega}(\bar{\theta}_n) - \Omega(\bar{\theta}_n) \| + \| \Omega(\bar{\theta}_n) - \Omega(\bar{\theta}) \| \\
        & \leq O_p(n^{-1/2}) + O_p(n^{-\zeta}) =  O_p(n^{-\zeta}).
    \end{align}
    
    In conclusion we have shown that $\widehat{\Omega}(\bar{\theta}_n)$ converges to the non-singular operator $\Omega(\bar{\theta})$ at rate $O_p(n^{-\zeta})$ and by Assumption~\ref{as6:first-estimate} we have $\lambda_n = O_p(n^{-\rho})$ with $0< \rho < \zeta$, therefore the operator $\widehat{\Omega}_{\lambda_n}(\bar{\theta}_n) = \widehat{\Omega}(\bar{\theta}_n) + \lambda_n I$ is non-singular with smallest eigenvalue bounded away from zero w.p.a.1.

    It remains to be shown that the largest eigenvalue of $\widehat{\Omega}(\bar{\theta}_n)$ is bounded. This is a direct consequence of Lemma~\ref{lemma:3:boundedness}. Consider any $h \in \mathcal{H}$ with $\|h \|_\mathcal{H} > 0$ and
    \begin{align}
        \langle h, \widehat{\Omega}(\bar{\theta}_n)  h \rangle &= E_{\ppn}[h(Z)^T \jpsi(X;\bar{\theta}_n) \jpsi(X;\bar{\theta}_n)^T h(Z)] \\
        &\leq E[\jpsi(X;\bar{\theta}_n)\|_\infty^2] E[\|h(Z)\|^2_\infty] \\
        &\leq L_\psi^2 C_h^2 < \infty.
    \end{align}
\end{proof}

\subsubsection{Proof of Theorem~\ref{th:3:consistency}}
\begin{lemma} \label{lemma:3:donsker-psi}
    Let the sets of functions $\{ \psi(\cdot;\theta)_l : \theta \in \Theta, l \in [m] \}$ and $H_1$ be $P_0$-Donsker. Then we have for any $\theta \in \Theta$
    \begin{align}
        \| E_{\ppn}[ \Psi(\xi;\theta)]  - E[ \Psi(\xi;\theta) ] \|_{\mathcal{H}^\ast} = O_p(n^{-1/2}).
    \end{align}
\end{lemma}
\begin{proof}
    \begin{align}
        \| E_{\ppn}[ \Psi(\xi;\theta)]  - E[ \Psi(\xi;\theta) ] \|_{\mathcal{H}^\ast} =& \sup_{h \in \mathcal{H}_1} E_{\ppn}[\psi(X;\theta)^T h(Z) ] - E[\psi(X;\theta)^T h(Z) ] \\
        =& \sup_{g \in \mathcal{G}} E_{\ppn}[g(\xi)] - E[g(\xi)]
    \end{align}
    where
    \begin{align}
        \mathcal{G} = \{g : g(\xi) = \sum_{i=1}^m \psi_i(x;\theta) h_i(z), \ h_i \in \mathcal{H}_{i,1}, \theta \in \Theta \}.
    \end{align}
    Now as each $h_i$ and $\psi_i(\cdot;\theta)$ are $P_0$-Donsker by Assumption~\ref{as1:instrument} and \ref{as4:regularity} respectively and uniformly bounded by Lemma~\ref{lemma:3:boundedness}, we can employ Lemma~\ref{lemma:3:donsker} to conclude that $\mathcal{G}$ is $P_0$-Donsker. From this, the result follows by application of Lemma~\ref{lemma:3:bennett}. 
\end{proof}

\begin{lemma}[Convergence of $\widehat{D}$]\label{lemma:3:rate-h}
    Let the assumptions of Theorem~\ref{th:3:consistency} be satisfied. Additionally let $\tilde{\theta} \in \Theta$ be a consistent estimator for $\theta_0$, i.e., $\tilde{\theta} \overset{p}{\rightarrow} \theta_0$ with $\| E_{\ppn}[\Psi(\xi;\tilde{\theta})] \|_{\mathcal{H}^\ast} = \op{n^{-1/2}}$. Then for $\tilde{h} = \argmax_{h \in \mathcal{H}} D(\tilde{\theta}, h)$ we have $\| \tilde{h} \|_\mathcal{H} = \op{n^{-1/2}}$ and $\widehat{D}(\tilde{\theta}, \tilde{h}) \leq \op{n^{-1}}$.
\end{lemma}
\begin{proof}
    Let $\tilde{\Psi} := \frac{1}{n} \sum_{i=1}^n \Psi(\xi_i,\tilde{\theta})$. Then we have
    \begin{align}
        0 &= \widehat{D}(\tilde{\theta},0) \\
        &\leq \argmax_{h \in \mathcal{H}} D(\tilde{\theta},\tilde{h}) \\
        &= \left(I + \frac{\epsilon}{2} \Delta_\xi \right) \tilde{\Psi}(\tilde{h}) - \frac{\epsilon}{2} \langle \tilde{h}, \widehat{\Omega}_{\lambda_n}(\bar{\theta}_n) \tilde{h} \rangle_\mathcal{H} \\
        &\leq \|I + \frac{\epsilon}{2} \Delta_\xi \|_\mathrm{op} \| \tilde{\Psi} \|_{\mathcal{H}^\ast} \|\tilde{h} \|_\mathcal{H}  - \frac{\epsilon}{2} \lambda_{\mathrm{min}}\left(\widehat{\Omega}_{\lambda_n}(\bar{\theta}_n)\right) \hnorm{\tilde{h}}^2 \\
        &\leq \left( 1 + \frac{\epsilon}{2}\| \Delta_\xi \| \right) \| \tilde{\Psi} \|_{\mathcal{H}^\ast} \|\tilde{h} \|_\mathcal{H}  - \frac{\epsilon}{2} \lambda_{\mathrm{min}}\left(\widehat{\Omega}_{\lambda_n}(\bar{\theta}_n)\right) \hnorm{\tilde{h}}^2
    \end{align}
    Using that $\| \Delta_\xi \| < \infty$ by Lemma~\ref{lemma:3:boundedness} 
    and moreover $\lambda_{\mathrm{min}}\left(\widehat{\Omega}_{\lambda_n}(\bar{\theta}_n)\right)> 0$ by Lemma~\ref{lemma:3:covariance}, we get $\| \tilde{h} \|_\mathcal{H} \leq C \| \tilde{\Psi} \|_{\mathcal{H}^\ast}$ and thus $\| \tilde{h} \|_\mathcal{H} = O_p\left(n^{-1/2} \right)$. Now inserting back into $\widehat{D}$ we get $\widehat{D}(\tilde{\theta}, \tilde{h}) \leq \op{n^{-1}}$.
    \end{proof}

\begin{lemma}[Convergence of $\|\hat{\Psi}\|_{\mathcal{H}^\ast}$] \label{lemma:3:rate-psi}
    Let the assumptions of Theorem~\ref{th:3:consistency} be satisfied. Let $\hat{\theta} = \argmin_{\theta \in \Theta} \sup_{h \in \mathcal{H}} \widehat{D}(\theta,h)$ denote the SMM estimator for $\theta_0$. Then $\hdnorm{E_{\ppn}[\Psi(\xi;\hat{\theta})]} = O_p(n^{-1/2})$.
\end{lemma}
\begin{proof}
    Let $\hat{\Psi} = \frac{1}{n} \sum_{i=1}^n \Psi(\xi,\hat{\theta})$. 
    Let $\phi(\hat{\Psi}) \in \mathcal{H}$ denote the Riesz representer of $\hat{\Psi} \in \mathcal{H}^\ast$. 
    Consider any $\sigma_n \rightarrow 0$ and define $h_{\hat{\Psi}} = \sigma_n \phi(\hat{\Psi})$.
     Using that the eigenvalues of the Laplacian $\Delta_\xi$ are bounded by Lemma~\ref{lemma:3:boundedness} and the largest eigenvalue of $\widehat{\Omega}(\bar{\theta}_n)$ is bounded by a constant $C$ by Lemma~\ref{lemma:3:covariance}, we have
    \begin{align}
    \widehat{D}(\hat{\theta}, h_{\hat{\Psi}}) &= \left(I + \frac{\epsilon}{2} \Delta_\xi \right) \hat{\Psi}(h_{\hat{\Psi}}) - \frac{\epsilon}{2} \langle h_{\hat{\Psi}}, \widehat{\Omega}(\bar{\theta}_n) h_{\hat{\Psi}} \rangle_\mathcal{H} \\
     &\geq \left(1 + \frac{\epsilon}{2} \lambda_\mathrm{min}(\Delta_\xi) \right) \hat{\Psi}(h_{\hat{\Psi}}) - \frac{\epsilon}{2} C \|h_{\hat{\Psi}} \|_\mathcal{H}^2 \\
     &\geq  C' \sigma_n \| \hat{\Psi} \|_{\mathcal{H}^\ast}^2 - \frac{C \epsilon}{2} \sigma_n^2 \| \hat{\Psi} \|_{\mathcal{H}^\ast}^2,
    \end{align}
    where by assumption on $\epsilon$ we have $C' = 1 + \frac{\epsilon}{2} \lambda_\mathrm{min}(\Delta_\xi) \neq 0$ w.p.1.
    Now, as $\hat{\theta}$ is the minimizer of the Sinkhorn profile $R(\theta) = \max_{h \in \mathcal{H}} \widehat{D}(\theta, h)$ we have 
    \begin{align}
       C' \sigma_n \| \hat{\Psi} \|_{\mathcal{H}^\ast}^2 - \frac{C \epsilon}{2} \sigma_n^2 \| \hat{\Psi} \|_{\mathcal{H}^\ast}^2 \leq \widehat{D}(\hat{\theta}, h_{\hat{\Psi}}) \leq \widehat{D}(\hat{\theta}, \hat{h}) \leq \max_{h \in \mathcal{H}} \widehat{D}(\theta_0, h) \leq O(n^{-1}),
    \end{align}
    where in the last step we used that $\| E_{\ppn}[\Psi(\xi;\theta_0)] \|_{\mathcal{H}^\ast}= O_p(n^{-1/2})$ by Lemma~\ref{lemma:3:donsker-psi} and thus the assumptions of Lemma~\ref{lemma:3:rate-h} are fulfilled and we get $\max_{h \in \mathcal{H}} \widehat{D}(\theta_0, h) \leq O(n^{-1})$.
    Thus we have $\sigma_n(C' - \frac{C\epsilon}{2} \sigma_n) \| \hat{\Psi} \|_{\mathcal{H}^\ast}^2 = O_p(n^{-1})$ and as $(C' - \frac{C\epsilon}{2} \sigma_n)$ is bounded away from zero for all $n$ large enough, we have $\sigma_n \| \hat{\Psi} \|_{\mathcal{H}^\ast}^2\leq O_p(n^{-1})$. As this holds for any $\sigma_n \overset{p}{\rightarrow} 0$ we finally have $\| \hat{\Psi} \|_{\mathcal{H}^\ast} = O_p(n^{-1/2})$.
\end{proof}

\paragraph{Proof of Theorem~\ref{th:3:consistency}}
Using the result of Lemma~\ref{lemma:3:rate-psi} for the convergence rate of the empirical moment functional, the proof of the consistency of our SMM estimator is identical to the ones provided by \citet{kremer2022functional} and \citet{pmlr-v202-kremer23a} for their estimtators. We provide it here for completeness.
\begin{proof}
    From Lemma~\ref{lemma:3:donsker-psi} it follows that $\| E_{\ppn}[\Psi(\xi;\theta)] - E [\Psi(\xi;\theta)] \|_{\mathcal{H}^\ast} = O_p(n^{-1/2})$ for any $\theta \in \Theta$. By Lemma~\ref{lemma:3:rate-psi} we have $\| E_{\ppn}[\Psi(\xi;\hat{\theta})] \|_{\mathcal{H}^\ast} = O_p(n^{-1/2})$ and thus using the triangle inequality we get
    \begin{align*}
        \left\| E[\Psi(\xi;\hat{\theta})] \right\|_{\mathcal{H}^\ast} &= \left\| E[\Psi(\xi;\hat{\theta})] - E_{\ppn}[\Psi(\xi;\hat{\theta})]  +  E_{\ppn}[\Psi(\xi;\hat{\theta})] \right\|_{\mathcal{H}^\ast} \\
        & \leq \left\| E[\Psi(\xi;\hat{\theta})] -  E_{\ppn}[\Psi(\xi;\hat{\theta})]  \right\|_{\mathcal{H}^\ast} + \left\| E_{\ppn}[\Psi(\xi;\hat{\theta})] \right\|_{\mathcal{H}^\ast} \\
        &= O_p(n^{-1/2}) \overset{p}{\rightarrow} 0.
    \end{align*}
    As by Assumption~\ref{as2:psi}, $\theta_0$ is the unique parameter for which $E[\psi(T,Y;\theta)|Z] = 0 \ P_z\mathrm{-a.s.}$ and by Assumption~\ref{as1:instrument} this is fulfilled if and only if $\| E[\psi(\xi;\theta)] \|_{\mathcal{H}^\ast} = 0 $, it follows that $\hat{\theta} \overset{p}{\rightarrow} \theta_0$. 
    
    Under the additional Assumption~\ref{as10:smooth-param} we can use this result to translate the convergence rate of the moment functional to a convergence rate of the estimator $\hat{\theta}$. 
    
    As $\Psi(\xi;\theta)$ is continuously differentiable in its second argument which follows immediately from Assumption~\ref{as10:smooth-param} and the definition of $\Psi$, we can use the mean value theorem to expand $\Psi(\xi,\hat{\theta})$ about $\theta_0$, i.e., there exists $\bar{\theta} \in \operatorname{conv}(\{\theta_0, \hat{\theta} \})$ such that 
    \begin{align}
        \Psi(\xi;\hat{\theta}) = \Psi(\xi;\theta_0) + (\hat{\theta} - \theta_0)^T \nabla_\theta \Psi(\xi;\bar{\theta}).
    \end{align}
    Using this we have
    \begin{align}
        \|E[\Psi(\xi;\hat{\theta})] \|^2_{\mathcal{H}^\ast} & = \| \underbrace{E[\Psi(\xi;\theta_0)]}_{=0} + (\hat{\theta} - \theta_0)^T E[ \nabla_\theta \Psi(\xi;\bar{\theta})]\|^2_{\mathcal{H}^\ast} \\
        &= \left\langle (\hat{\theta} - \theta_0)^T E[ \nabla_\theta \Psi(\xi;\bar{\theta})], (\hat{\theta} - \theta_0)^T E[ \nabla_\theta \Psi(\xi;\bar{\theta})] \right\rangle_{\mathcal{H}^\ast} \\
        &= (\hat{\theta} - \theta_0)^T \underbrace{\left\langle  E[ \nabla_\theta \Psi(\xi;\bar{\theta})], E[ \nabla_{\theta^T} \Psi(\xi;\bar{\theta})] \right\rangle_{\mathcal{H}^\ast}}_{=: \Sigma(\bar{\theta})} (\hat{\theta} - \theta_0) \\
        &\geq \lambda_\text{min}\left(\Sigma(\bar{\theta})\right) \| \hat{\theta} - \theta_0 \|^2_2.
    \end{align}
    Now as $\hat{\theta} \overset{p}{\rightarrow} \theta_0$ and $\bar{\theta} \in \operatorname{conv}(\{\theta_0, \hat{\theta} \})$ we have $\bar{\theta} \overset{p}{\rightarrow} \theta_0$ and thus $\Sigma(\bar{\theta}) \overset{p}{\rightarrow} \Sigma(\theta_0) =: \Sigma_0$ by the continuous mapping theorem. By the non-negativity of the norm $\Sigma_0$ is positive-semi definite and non-singular by Lemma~\ref{lemma:3:kremer}, thus the smallest eigenvalue of $\Sigma(\bar{\theta})$, $\lambda_\text{min}(\Sigma(\bar{\theta}))$, is positive and bounded away from zero w.p.a.1. Finally as $\|E[\Psi(X,Z;\hat{\theta})] \| = O_p(n^{-1/2})$ taking the square-root on both sides we have $\|\hat{\theta} - \theta_0 \| = O_p(n^{-1/2 })$.
\end{proof}

\paragraph{Proof of Proposition~\ref{prop:3:rkhs}}
\begin{proof}
    For a universal ISPD kernel, equivalence of the conditional and the variational moment restrictions \eqref{eq:3:cmr} and \eqref{eq:3:vmr} follows by Theorem~3.9 of \citet{kremer2022functional}. The Donsker property of the unit ball in an RKHS of a smooth universal kernel with compact domain follows from Lemma~17 of \citet{bennett2020variational}. Finally, the Donsker property of the Jacobian $J_z h$ of $h$ follows by the same argument as Lemma~17 of \citet{bennett2020deep} using $C^{\infty}$ smoothness of $h$ and boundedness of $J_z h$.  
\end{proof}

\paragraph{Proof of Theorem~\ref{th:3:kernel-smm}}
\begin{proof}
    Under the assumptions the Sinkhorn profile is given as
    \begin{align}
        R_\lambda(f) = \sup_{h \in \mathcal{H}} &\Bigg\{ E_{\ppn}\left[h(Z)^T \left(I+\frac{\epsilon}{2} \Delta_x \right) \psi(X;f) \right] \\
        &- \epsilon E_{\ppn}\left[ h(Z)^T \jpsi(X;\tilde{f}) \Gamma_x^{-1} \jpsi(X;\tilde{f})^T h(Z) \right] - \frac{\lambda}{2} \| h\|_\mathcal{H}^2 \Bigg\} \label{eq:3:r}
    \end{align}
    which as the unconstrained maximization of a concave objective is a convex optimization problem. Moreover, the conditions of the classical representer theorem~\citep{Schoelkopf01:Representer} are fulfilled and thus the maximizer of \eqref{eq:3:r} is given as $h_l = \sum_{i=1}^n \alpha_i^l k_l(z_i,\cdot)$ with $\alpha^l \in \mathbb{R}^n$.
    Inserting this into \eqref{eq:3:r} and defining the kernel Gram matrices $K_l \in \mathbb{R}^{n\times n}$ with entries $(K_l)_{ij} = k_l(z_i, z_j)$ we obtain
    \begin{align}
        R_\lambda(f) =& \sup_{\alpha \in \mathbb{R}^{nm}} \frac{1}{n} \sum_{i,j=1}^n \sum_{l=1}^m \alpha^l_i (K_l)_{ij}   \left(I + \frac{\epsilon}{2} \Delta_x \right) \psi_l(x_j;f) - \frac{\lambda}{2} \sum_{l=1}^m (\alpha^l)^T K_l (\alpha^l) \\
        &- \frac{\epsilon}{2n} \sum_{i,j,k=1}^n \sum_{l,r=1}^m \alpha^l_i (K_l)_{ij} \nabla_x \psi_l(x_j;\tilde{f})^T  \Gamma_x^{-1} \nabla_x \psi_r(x_j;\tilde{f}) (K_r)_{jk} \alpha_r^k  \\
        =& \sup_{{\alpha} \in \mathbb{R}^{nm}} \frac{1}{n} \alpha^T L \psi_\Delta - \frac{1}{2} \alpha^T \left( \epsilon Q(\tilde{f}) + \lambda L \right) \alpha
    \end{align}
    where we defined $\psi_\Delta(f) \in \mathbb{R}^{nm}$, $L \in \mathbb{R}^{nm \times nm}$ and $Q(f) \in \mathbb{R}^{nm \times nm}$ with entries
    \begin{align}
        \psi_\Delta(f)_{i \cdot l} &= \left( I + \frac{\epsilon}{2} \Delta_x \right)\psi_l(x_i;f)  \\
        L_{(i\cdot l), (j \cdot r)} &= \delta_{lr}  k_l(z_i,z_j) \\ 
        Q(f)_{ (i\cdot l) , (j \cdot r)} &= \frac{1}{n} \sum_{k=1}^n \sum_{s=1}^{d_x} k_l(z_i, z_k) \nabla_{x_s} \psi_l(x_k;{f}) (\Gamma_x^{-1})_{ss} \nabla_{x_s} \psi_r(x_k; {f}) k_r(z_k, z_j).
    \end{align}
    The first order optimality conditions for $\alpha$ read
    \begin{align}
        0 = \frac{1}{n} L \psi_\Delta(f) - \left( \epsilon Q(\tilde{f}) + \lambda L \right) \alpha,
    \end{align} 
    which immediately gives
    \begin{align}
        \alpha = \left( \epsilon Q(\tilde{f}) + \lambda L \right)^{-1}  \frac{1}{n} L \psi_\Delta(f).
    \end{align}
    Inserting back into $R_\lambda(f)$ and multiplying by $\epsilon >0 $ we obtain
    \begin{align}
        R_\lambda(f) = \frac{1}{2n^2} \psi_\Delta(f)^T L \left(Q(\tilde{f}) + \frac{\lambda}{\epsilon} L \right)^{-1} L \psi_\Delta(f).
    \end{align}
\end{proof}

\end{document}